\documentclass{article}

\usepackage[main, final]{neurips_2026}

\usepackage[utf8]{inputenc} 
\usepackage[T1]{fontenc}    
\usepackage{hyperref}       
\usepackage{url}            
\usepackage{booktabs}       
\usepackage{amsfonts}       
\usepackage{nicefrac}       
\usepackage{microtype}      
\usepackage{xcolor}         

\usepackage{microtype}
\usepackage{graphicx}
\usepackage{subcaption}

\usepackage{amsmath}
\usepackage{amssymb}
\usepackage{mathtools}
\usepackage{amsthm}
\usepackage{comment}

\usepackage{float}

\theoremstyle{plain}
\newtheorem{theorem}{Theorem}[section]
\newtheorem{proposition}[theorem]{Proposition}
\newtheorem{lemma}[theorem]{Lemma}
\newtheorem{corollary}[theorem]{Corollary}
\theoremstyle{definition}
\newtheorem{definition}[theorem]{Definition}
\newtheorem{assumption}[theorem]{Assumption}
\theoremstyle{remark}

\usepackage{enumitem}

\newcommand{\dwl}[1]{}
\newcommand{\tha}[1]{}
\newcommand{\oscar}[1]{}
\newcommand{\MH}[1]{}

\usepackage{tcolorbox}

\newcommand{\ourname}{SGD-LL}

\title{Revisiting the Adam-SGD Gap in LLM Pre-Training: The Role of Large Effective Learning Rates}

\author{%
  Athanasios Glentis \qquad Dawei Li \qquad Chung-Yiu Yau \qquad Mingyi Hong\\ \\
   University of Minnesota\\
   \texttt{\{glent007,li004678,cyau,mhong\}@umn.edu}
}

\begin{document}

\maketitle

\begin{abstract}

It is widely believed that stochastic gradient descent (SGD) performs significantly worse than adaptive optimizers such as Adam in pre-training Large Language Models (LLMs). 
Yet the underlying reason for this gap remains unclear. In this work, we attribute a large part of the discrepancy to SGD's inability to sustain learning rates comparable to Adam's much larger effective learning rates. 
Through empirical and theoretical analysis of LLM pre-training dynamics, we identify that training is characterized by small gradient norms and large weight-to-gradient ratios, an effect that becomes more pronounced with larger batch sizes typical in pre-training, necessitating such large effective learning rates. 
However, we find that output-layer gradient magnitudes become highly uneven across token classes, and that large gradient spikes frequently occur during training. Together, these effects severely restrict the admissible learning rate of SGD.
Guided by this understanding, we show that simple clipping mechanisms that stabilize SGD at large learning rates enable it to recover most of Adam's performance. In our large-scale experiments, the validation loss gap between large-learning-rate SGD and Adam shrinks from more than 50\% to only about 2.3\% when pre-training a 7B-parameter LLaMA model with a 1M-token batch size. Code is available at this \href{https://github.com/OptimAI-Lab/large_lr_sgd}{\textcolor{blue}{link}}.

\end{abstract}

\begin{figure}[h]
    \centering
    \begin{subfigure}[t]{0.32\linewidth}
        \centering
        \includegraphics[width=\linewidth]{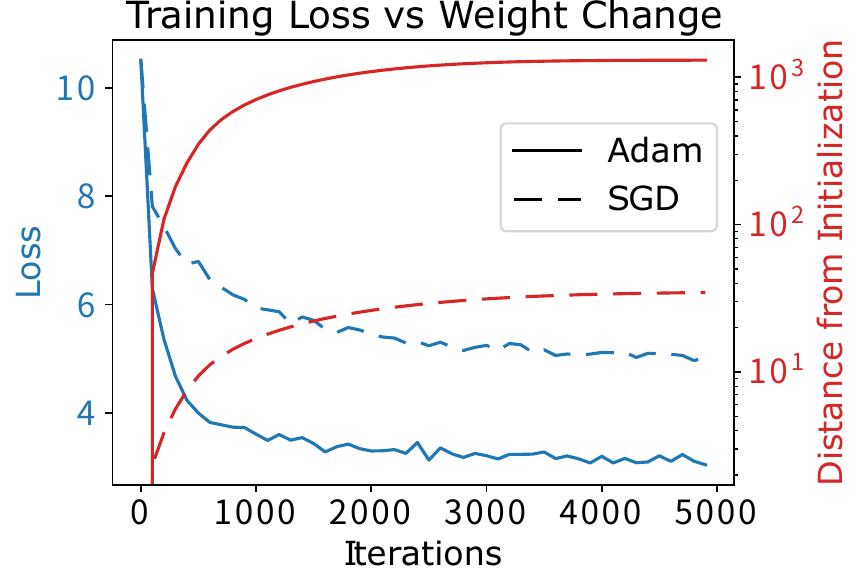}
        \caption*{(a)}
    \end{subfigure}
    \begin{subfigure}[t]{0.32\linewidth}
        \centering
        \includegraphics[width=\linewidth]{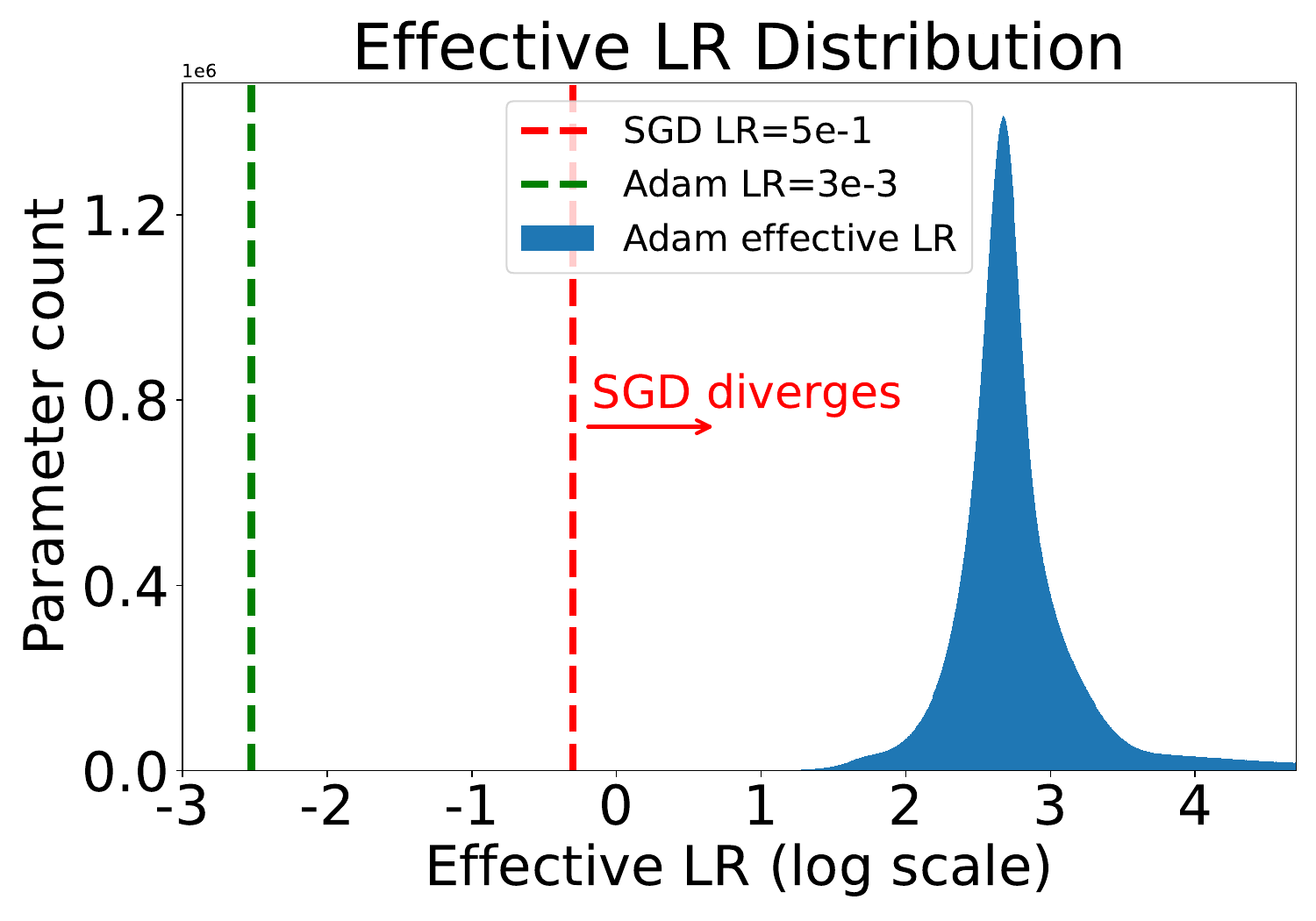}
        \caption*{(b)}
    \end{subfigure}
    \begin{subfigure}[t]{0.32\linewidth}
        \centering
        \includegraphics[width=\linewidth]{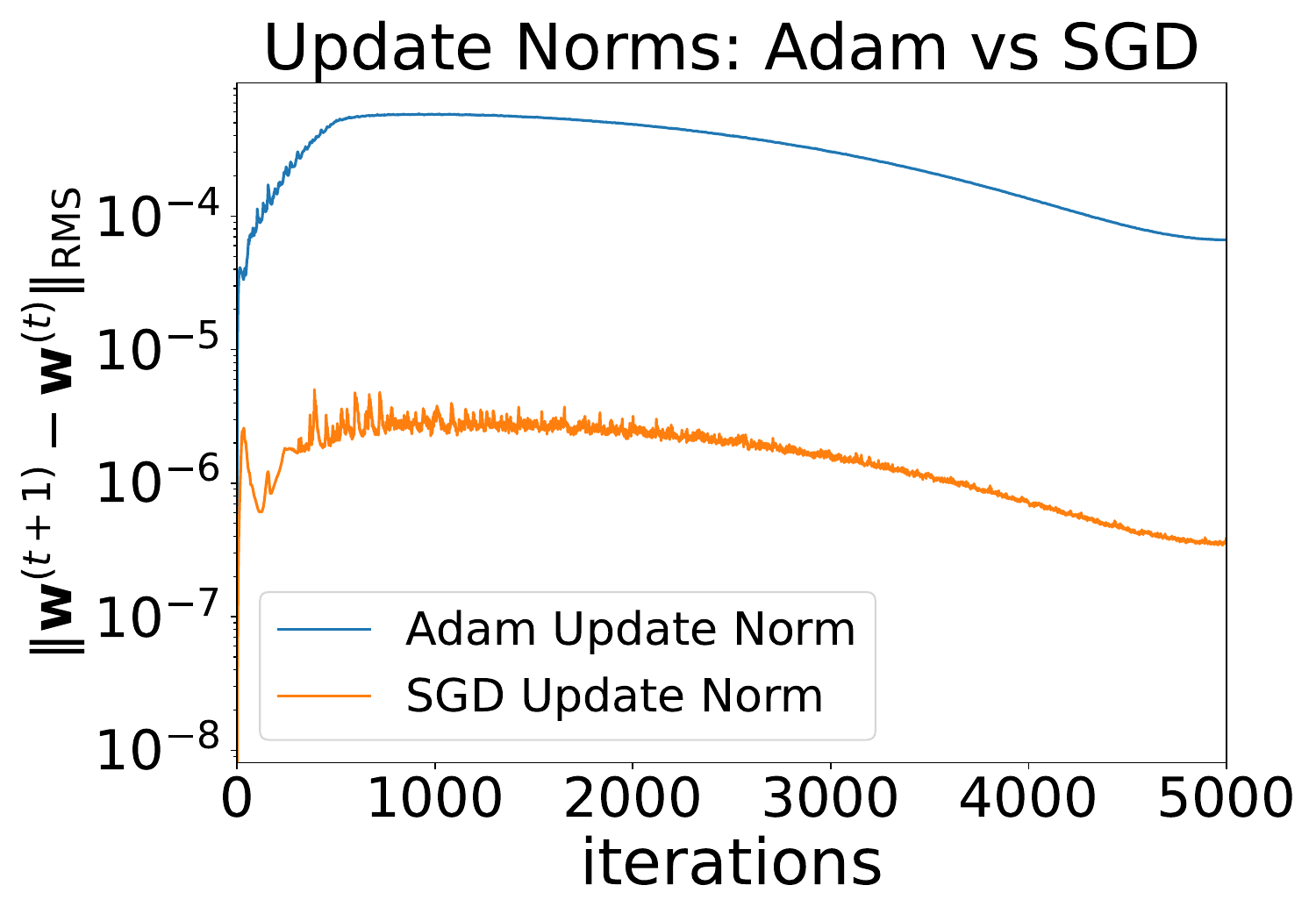}
        \caption*{(c)}
    \end{subfigure}
    \caption{
    (a) The pre-training dynamics of Adam and (momentum) SGD: (Left vertical axis) Training loss. (Right vertical axis) Distance from initialization $\| {\bf w}^{(t)} - {\bf w}^{(0)} \|_F$. The effective solution space of SGD is about $50\times$ smaller than that of Adam.   %
    (b) Effective learning rates (LR) $a^{(t)}_i$ at iteration $t=500$. For SGD $a^{(t)}_i\equiv \eta_t$ for all parameters; for Adam $a^{(t)}_i=\eta_t/(\sqrt{v^{(t)}_{i}} + \varepsilon)$ for each $w_i^{(t)}$ (see Section \ref{sec:eff_lr} for details). The effective LR of Adam are 2-4 orders of magnitude larger than SGD's. Base LR 
    tuned to 3e-3 for Adam and 0.5 for SGD (dashed lines).  Increasing SGD's LR further leads to divergence.  
    (c) The update norm of Adam stays about 2 orders of magnitude larger than SGD's during training. 
   All figures are from training LLaMA 130M on C4 with a 0.5M token batch size. 
    \vspace{-10pt}
    }
    \label{fig:main_figs}
\end{figure}

\section{Introduction}

Adaptive optimization algorithms such as Adam \citep{kingma2015adam} have been empowering the optimization of large-scale LLM pre-training tasks. In contrast, the classical baseline algorithm (momentum) SGD \citep{robbin1951stochastic} has received significantly less attention in the literature of pre-training. For instance, SGD is no longer seen as a baseline in recent benchmarks of optimizers for pre-training \citep{wen2025fantastic, semenov2025benchmarking}. This points to the common belief that SGD is not suitable for the task of LLM pre-training.

Notably, the reported performance gap between SGD and Adam in standard pre-training settings is exceptionally large and consistently observed across studies \citep{pan2023toward, kunstner2024heavy, zhang2024transformers, marek2025small}. The same  gap is also reproduced in our Figure \ref{fig:main_figs} (a), where SGD results in over 60\% final loss increase relative to Adam. This indicates that SGD plateaus quickly in a high loss region, failing to drive the loss to acceptable levels. 

Taking this gap as given, prior works have focused considerable effort on explaining why SGD is ineffective for pre-training. Among others, \citet{kunstner2024heavy} studied the heavy-tailed class imbalanced data distribution in pre-training, observing that SGD struggles to progress on infrequent targets. \citet{pan2023toward} analyzed the empirical local geometry of the pre-training loss and suggested that SGD has ill-conditioned sharpness. \citet{zhang2024transformers} observed Hessian block heterogeneity to be coexisting with the poor performance of SGD. 

While these prior results are highly important, providing valuable insights into specific aspects of SGD in pre-training, the broader training dynamics and their relation to SGD's poor performance remain significantly less explored. Thus, in this work, rather than focusing on isolated behaviors of SGD, we examine closely its broader, high-level  ``symptoms'' tied to pre-training dynamics that lead to the considerable gap with Adam.  

In particular, we begin by observing that SGD, even after extensive learning rate tuning, only explores a small neighborhood around initialization, considerably smaller than Adam (as shown in Figure \ref{fig:main_figs} (a)). This phenomenon is not immediately apparent, as Adam typically uses a much smaller (base) learning rate than SGD. This motivates us to examine their \textit{effective} learning rates, i.e., the scaling each method applies to the momentum of the gradient, where we observe a striking difference: the effective learning rates of Adam are orders of magnitude larger than the learning rate of SGD (Figure \ref{fig:main_figs} (b)). Consequently, the overall update strength of the two methods differs substantially
(Figure \ref{fig:main_figs} (c)), 
suggesting that SGD’s smaller updates are insufficient to keep pace. We further argue that pre-training dynamics, characterized by small gradient norms and large weight-to-gradient ratios, necessitate such large effective learning rates, which we support theoretically and relate to the typical large batch size used in pre-training. 

To further investigate this gap, we next ask \textit{why} the learning rate in SGD cannot simply be increased further. We find that, at larger learning rates gradient irregularities become a key limiting factor for SGD due to its unconstrained update. Motivated by this observation, we validate our conjecture by ``treating'' these irregularities via simple, yet targeted gradient clipping, which enables us to substantially increase the SGD learning rate and close most of the gap to Adam. 

We emphasize that our perspective is fresh and complementary to prior work: while existing studies explain specific mechanisms underlying SGD’s behavior, our focus is on high-level training dynamics that account for a large part of the observed discrepancy, leaving finer-grained factors to explain the remaining gap.

We summarize the main contributions of our work as follows:  

\vspace{-10pt}

\begin{itemize}[leftmargin=10pt]
    \item We empirically study the effective learning rates of Adam compared to that of SGD in the setting of LLM pre-training. Even after tuning both algorithms' base learning rates, we find a substantial discrepancy, i.e., Adam’s effective learning rates are 2-4 orders of magnitude larger than SGD's. We further observe the consequence of that, being insufficient update norm for SGD and thus limited exploration of the solution space (see Figure \ref{fig:main_figs} for details). 
    \item We identify that the need for high learning rates in LLM pre-training arises from the training dynamics, in particular the substantially small gradient norms observed. This results in a higher weight-to-stochastic-gradient (weight-to-SG) norm ratio than that observed in other training setups, e.g., convolutional network training \citep{you2017large}. We further observe that this effect intensifies with increasing the batch size.
    \item Theoretically, we provide an explanation for the small stochastic gradient norms observed in LLMs. In particular, by decomposing the per-token-class gradients under two complementary regimes (i.e., bias-variance decomposition \eqref{eq:tokenwise_grad_batch_sec_mainbody} and background-hit decomposition \eqref{eq:tokenwise_grad_rare_mainbody}) we derive a universal upper bound for the norm of the stochastic gradient (see Theorem \ref{thm:full_grad_mainbody_unified}). We further analyze the terms in the upper bound, showing that under assumptions in the training dynamics and token distribution, this bound decreases with batch size, consistent with our empirical observations. 
    \item We identify two types of gradient irregularities in pre-training that contribute to SGD's divergence at high learning rates. First, we observe frequent layer-wise gradient norm spikes especially at the early stage of training. Second, at the output layer of size ${d\times |V|}$ where $|V|$ is the number of the output tokens, we observe that token-class SG norms $\| {\bf g}_1 \|_2, ..., \| {\bf g}_V \|_2$ are highly imbalanced. 
    \item Finally, we demonstrate that enabling larger learning rates in SGD can close most of the observed performance gap to Adam in pre-training. To achieve this we use two simple, yet targeted, types of  gradient clipping to ``treat'' the aforementioned gradient irregularities. To our knowledge, this is the first time SGD has been successfully used to pre-train LLMs with such large learning rates (100–300 in our experiments), substantially narrowing the gap to Adam. Specifically, in our largest experiment pre-training LLaMA 7B on C4 with SGD having learning rate 300 and a 1M-token batch size, the validation loss gap to Adam closes from over 50\% to only about 2.3\%.  
\end{itemize}

\vspace{-5pt}

\section{Preliminaries and Related Works}

\vspace{-5pt}
 
The \textbf{LLM pre-training objective} can be expressed as:
\begin{equation}\label{eq:finite_sum}
    \min_{\mathbf{w}} \ell(\mathbf{w}):=\frac{1}{n}\sum_{i=1}^{n}\ell(\mathbf{w};\xi_i)
\end{equation}
where $\ell$ is the loss function typically defined as the cross entropy loss of predicting the next token, $\mathbf{w}$ are the trainable parameters of the model and $\xi_i$, $i=1,...,n$ are our training input sequences.

To solve \eqref{eq:finite_sum}, \textbf{SGD}, at every training iteration $t$, draws a batch of samples $\{\xi_{t,b}\}_{b=1,..,B}$ and performs an update toward the negative stochastic gradient direction:
\begin{equation}
    \mathbf{w}^{(t+1)} = \mathbf{w}^{(t)} - \eta_t \mathbf{g}^{(t)}, \quad {\rm with}\quad \mathbf{g}^{(t)}:=\frac{1}{B}\sum_{b=1}^{B}\nabla \ell(\mathbf{w}^{(t)};\xi_{t, b})
\end{equation}
Here $B$ is batch size and $\eta_t$ is the learning rate at iteration $t$.  A straightforward extension is momentum SGD (which we will simply refer to as SGD throughout), replacing the gradient with its momentum to smooth the updates:
\begin{equation}
    \mathbf{m}^{(t)} = \beta \mathbf{m}^{(t-1)} + (1 - \beta) \mathbf{g}^{(t)} \quad
    \mathbf{w}^{(t+1)} = \mathbf{w}^{(t)} - \eta_t \mathbf{m}^{(t)}.
\end{equation}

 In contrast, \textbf{Adam}~\citep{kingma2015adam} updates the parameters using a more sophisticated scheme as shown below ($\odot$ is element-wise product; bias correction term omitted for simplicity):

\begin{equation}\label{eq:adam_update}
\begin{array}{cc}
\begin{aligned}
&\mathbf{m}^{(t)} = \beta_1 \mathbf{m}^{(t-1)} + (1 - \beta_1)\mathbf{g}^{(t)}, \\
&\mathbf{v}^{(t)} = \beta_2 \mathbf{v}^{(t-1)} + (1 - \beta_2) \mathbf{g}^{(t)} \odot \mathbf{g}^{(t)}
\end{aligned}
&
\mathbf{w}^{(t+1)} =
\mathbf{w}^{(t)} -
\eta_t \frac{\mathbf{m}^{(t)}}{\sqrt{\mathbf{v}^{(t)}} + \epsilon}.
\end{array}
\end{equation}

Note that, unless explicitly stated otherwise, we use the standard hyperparameters $\beta=0.9$ for SGD and $(\beta_1=0.9, \beta_2=0.95)$ for Adam in all our experiments.

\textbf{Related works.} It was commonly observed that SGD is far behind adaptive optimizers such as Adam on typical LLM pre-training problems \citep{pan2023toward, kunstner2024heavy, zhang2024transformers, marek2025small,sreckovic2025your}, while such large gap is not typically observed in other training settings \citep{keskar2017improving,zhang2020adaptive}. 

One line of works tries to attribute such gap to the text data distribution. Notably, \citet{zhang2020adaptive} 
first argued that heavy-tailed stochastic gradient noise in Wikipedia corpus BERT pre-training causes SGD non-convergence, while well-concentrated stochastic gradient noise in Imagenet Resnet training does not. Their solution depends on an adaptive step size that is at the same complexity as Adam.
In a related work, \citet{kunstner2024heavy} argues that LLMs are pre-trained on class-imbalanced data distributions, implying that heavy-tail rare tokens contribute to heavy-tail large stochastic gradient noise and SGD fails to capture the rare tokens' gradient that deviates from the momentum.

Another line of works focuses on the geometry of LLM optimization. \citet{pan2023toward} suggests that SGD has high directional sharpness as compared to Adam, which results to slow convergence. Meanwhile, \citet{zhang2024transformers} proposed that block heterogeneous Hessian prevents SGD from achieving good performance. 

More recently, \citet{sreckovic2025your} point out that the performance gap between SGD and Adam narrows in the very small batch size setting, debunking the prior explanations on this performance gap. Similarly, \citet{marek2025small} suggests that training with very small batch sizes (as low as batch size 1) displays minimal sensitivity to other training hyperparameters and choice of optimizer. Our insights will further support those previous works, since we will show that the pre-training phenomena which hinder SGD's performance become worse as batch size increases.

\textbf{Our approach.}  The main focus of this work is to study the high-level pre-training dynamics, including effective learning rates and weight to gradient norm ratios, combining empirical and theoretical insights. This analysis will suggest that SGD needs much larger base learning rate to make effective training progress, which can be enabled by ``treating'' gradient irregularities and close \textit{most} of the gap with Adam. Thus, we complement prior works which focus on  specific limitations of SGD in pre-training that combined may explain the remaining portion of the gap.


 \vspace{-5pt}

\section{LLM Training Dynamics and the Need for Large Learning Rate}

 \vspace{-5pt}

In this section, we explore the LLM training dynamics. In particular, we observe empirically that SGD has a much smaller ``effective'' learning rate compared to Adam. Then, both empirically and analytically, we provide justifications for the necessity for a large ``effective'' learning rate in LLM pre-training. The empirical observations here are from pre-training a LLaMA 130M model \citep{touvron2023llama} on C4 \citep{raffel2020exploring} with token batch size 0.5M (see Appendix \ref{apx:experimental_settings} for more details).

\vspace{-5pt}

\subsection{Effective Momentum Learning Rate}

\vspace{-5pt}

\label{sec:eff_lr}

To begin with, we justify the difference in the magnitude of update between different optimizers. To quantify the magnitude of update, we focus on the case of stochastic momentum algorithms and define {\bf effective momentum learning rate} as follows: 

\begin{definition} \label{def:eff_mom_ss}
    For stochastic momentum optimization algorithms under the following update rule for model parameters
    ${\bf w}^{(t)} = [w_1^{(t)}, ..., w_d^{(t)}] \in \mathbb{R}^d$ with arbitrary initialization ${\bf w}^{(0)}$, ${\bf a}^{(t)} = [a_1^{(t)}, ..., a_d^{(t)}] \ge {\bf 0}$, $\beta \in [0,1)$, ${\bf m}^{(0)} = \nabla f({\bf w}^{(0)}; \xi^{(0)})$: 
        \begin{subequations}
    \begin{align}
            {\bf w}^{(t+1)} &= {\bf w}^{(t)} - {\bf a}^{(t)} \odot {\bf m}^{(t)}, \\
            {\bf m}^{(t+1)} &= \beta {\bf m}^{(t)} + (1 - \beta) \nabla f\left({\bf w}^{(t+1)}; \xi^{(t+1)}\right),
    \end{align}
        \end{subequations}

\vspace{-.3cm}        
we define $a^{(t)}_i$ as the \textbf{effective momentum learning rate} for parameter $w_i^{(t)}$ (also referred as \textit{effective learning rate} for simplicity). 
\end{definition}
\vspace{-.15cm}

By Definition \ref{def:eff_mom_ss}, the effective learning rate of SGD is simply the learning rate $a^{(t)}_i\equiv \eta_t$ for all parameters $i \in \{1, ..., d\}$. On the other hand, by the Adam update \eqref{eq:adam_update}, the effective learning rate of Adam for $w_i^{(t)}$ is $a^{(t)}_i=\eta_t/(\sqrt{v^{(t)}_{i}} + \varepsilon)$. Note that $\eta_t$ follows the cosine learning rate schedule with warm-up and the base learning rate $\eta$ (i.e., the peak value of $\eta_t$) is tuned for both optimizers. 

We find that despite the small learning rate typically used for Adam (tuned to $3 \times 10^{-3}$ in our experiment), $a^{(t)}_i$ is significantly large on average as shown in Figure \ref{fig:eff_lr_and_update_norm} (a), ranging from around $10^2$ to $10^4$ throughout most of the training. In comparison, the same quantity for SGD peaks only to 0.5, while larger learning rates result to training divergence (similar to prior observations such as \citet{zhang2024transformers}).

\begin{figure}
    \centering
    \begin{subfigure}[t]{0.26\linewidth}
        \centering
        \includegraphics[width=\linewidth]{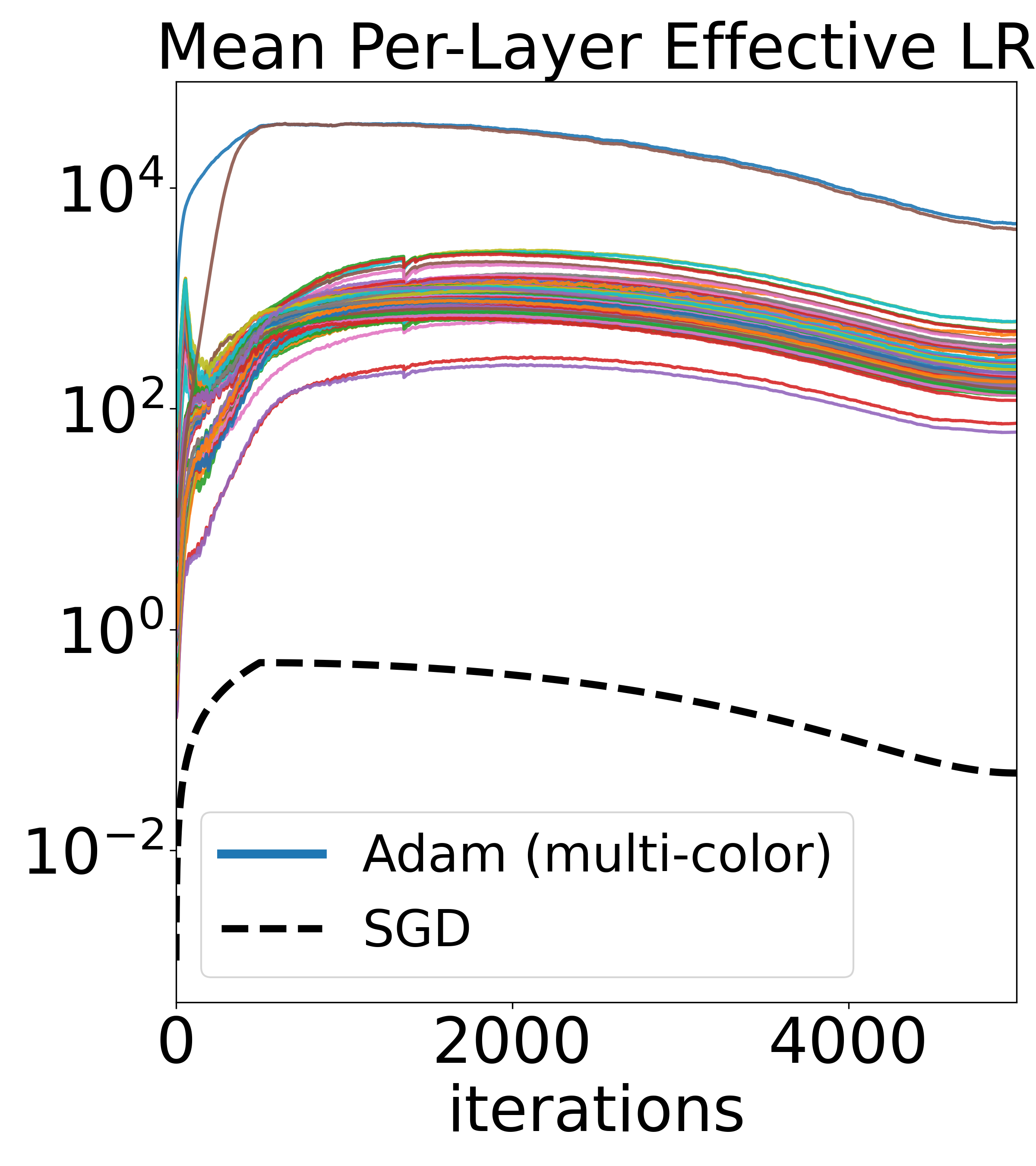}
        \caption*{(a)}
    \end{subfigure}
    \hfill
    \begin{subfigure}[t]{0.29\linewidth}
        \centering
        \includegraphics[width=\linewidth]{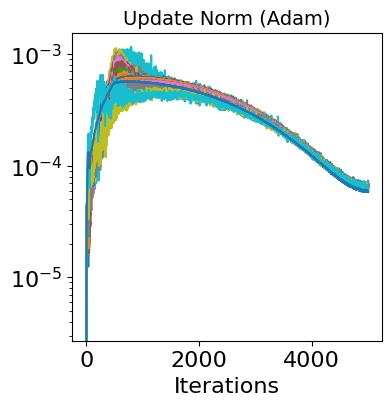}
        \caption*{(b)}
    \end{subfigure}
    \hfill
    \begin{subfigure}[t]{0.29\linewidth}
        \centering
        \includegraphics[width=\linewidth]{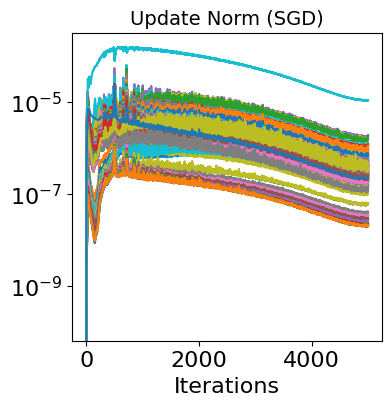}
        \caption*{(c)}
    \end{subfigure}
    \caption{(a) Mean effective learning rate per layer for Adam and SGD.  We use cosine scheduler with warmup and base learning rate tuned to 3e-3 for Adam and 0.5 for SGD. The Adam effective learning rates are 2-4 orders of magnitude larger than SGD's. (b-c) The dynamics of layer-wise update norm $\| \mathbf{w}^{(t+1)}_{\ell} - \mathbf{w}^{(t)}_{\ell} \|_{\rm RMS}$ during training. The update norm of SGD is over $100\times$ smaller than that of Adam. We use different colors to represent the averages for different layers.
    }
    \label{fig:eff_lr_and_update_norm}
    \vspace{-10pt}
\end{figure}

The small effective learning rate of SGD indicates that the update norm of SGD during training is also much smaller than Adam's. We observe such phenomenon as well in Figures \ref{fig:eff_lr_and_update_norm} (b-c). Such small update norm suggests that SGD has limited exploration behavior, as shown in Figure \ref{fig:main_figs} (a), restricting its ability to keep pace with Adam.

\subsection{Weight-gradient Norm Ratio}

\vspace{-5pt}

Given the aforementioned observation that Adam has a much larger effective learning rate than SGD in LLMs, a natural question is: whether such a large effective learning rate is necessary for obtaining good performances. In this section, we proceed to answer this question by observing the dynamics of model weight parameters and stochastic gradient (SG) via the weight-SG norm ratio $\| \mathbf{w} \|_F / \| \nabla f({\bf w}; \xi) \|_F$.

Intuitively, as the weight-SG norm ratio gets larger, the learning rate needs to also scale up to ensure sufficiently large update for the weights. 
Note also that, a large weight-SG norm ratio implies a large (or often even larger) weight-momentum norm ratio, as the momentum is simply a smoothed version of the SG. 
In earlier works on large batch training of convolutional neural network on ImageNet, this ratio is at the order of $10^0$ to $10^3$ \citep{you2017large}. Meanwhile in the setting of LLM pre-training, our measured weight-SG norm ratio is significantly larger, ranging from $10^1$ to $10^4$ at initialization, and $10^3$ to $10^5$ during training; see Figure~\ref{fig:w_sg_ratio_training}.

\begin{figure}
    \centering
    \begin{subfigure}[t]{0.25\linewidth}
        \centering
        \includegraphics[width=\linewidth]{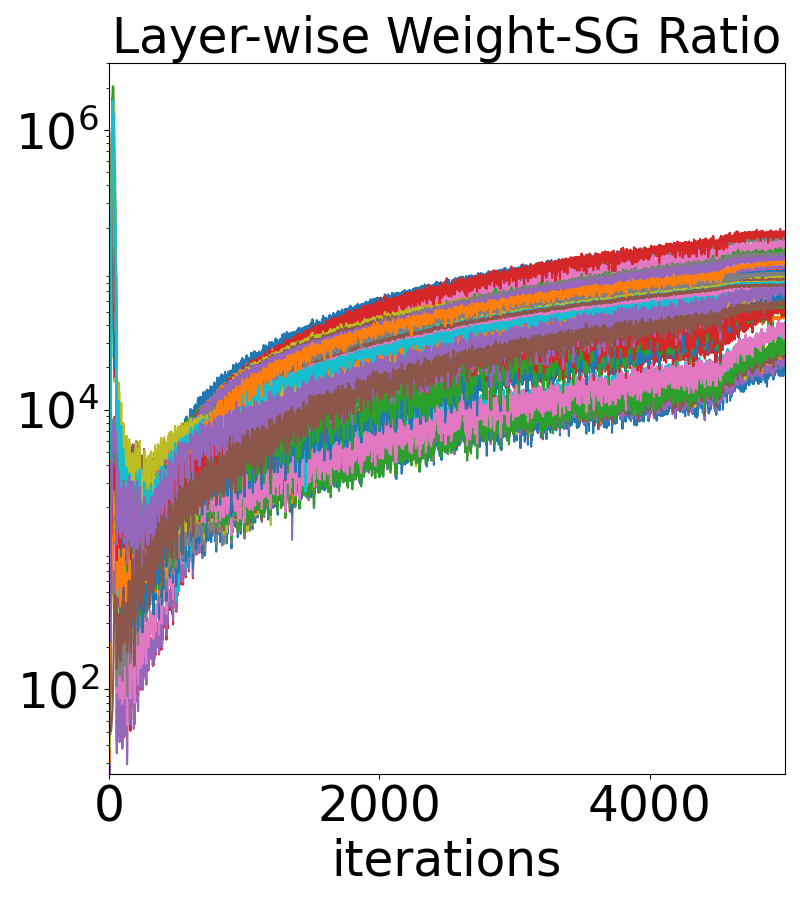}
        \caption*{(a)}
    \end{subfigure}
    \hspace{10pt}
    \begin{subfigure}[t]{0.28\linewidth}
        \centering
        \includegraphics[width=\linewidth]{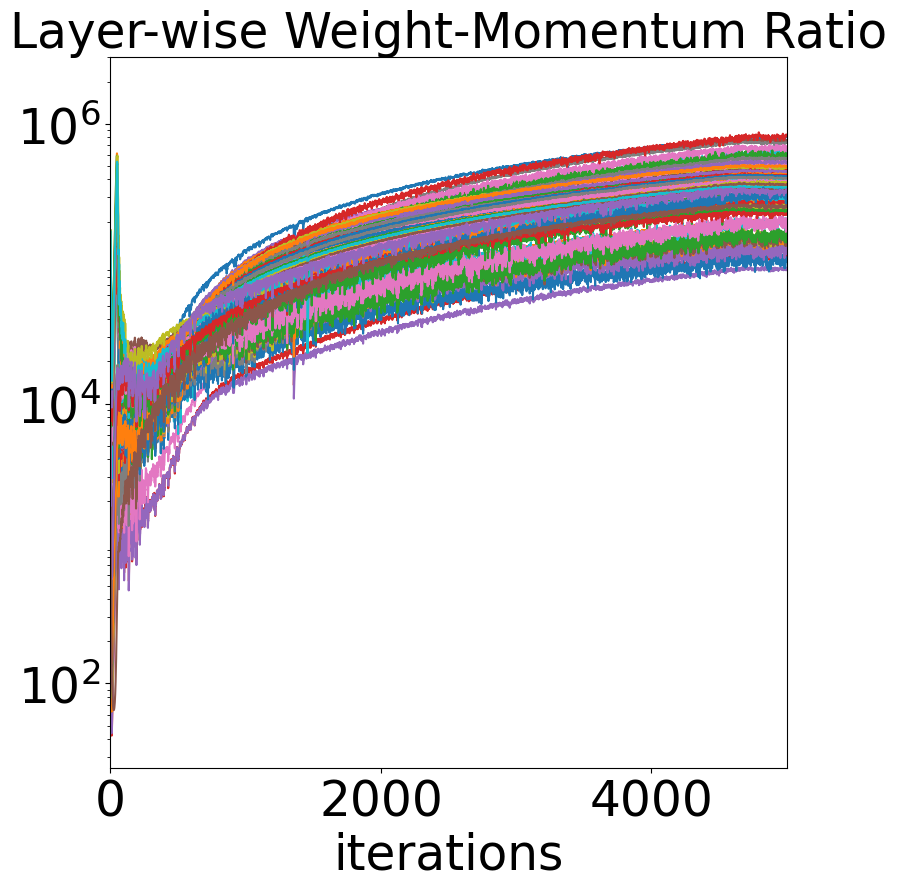}
        \caption*{(b)}
    \end{subfigure}
    \hspace{10pt}
    \begin{subfigure}[t]{0.38\linewidth}
        \centering
        \includegraphics[width=\linewidth]{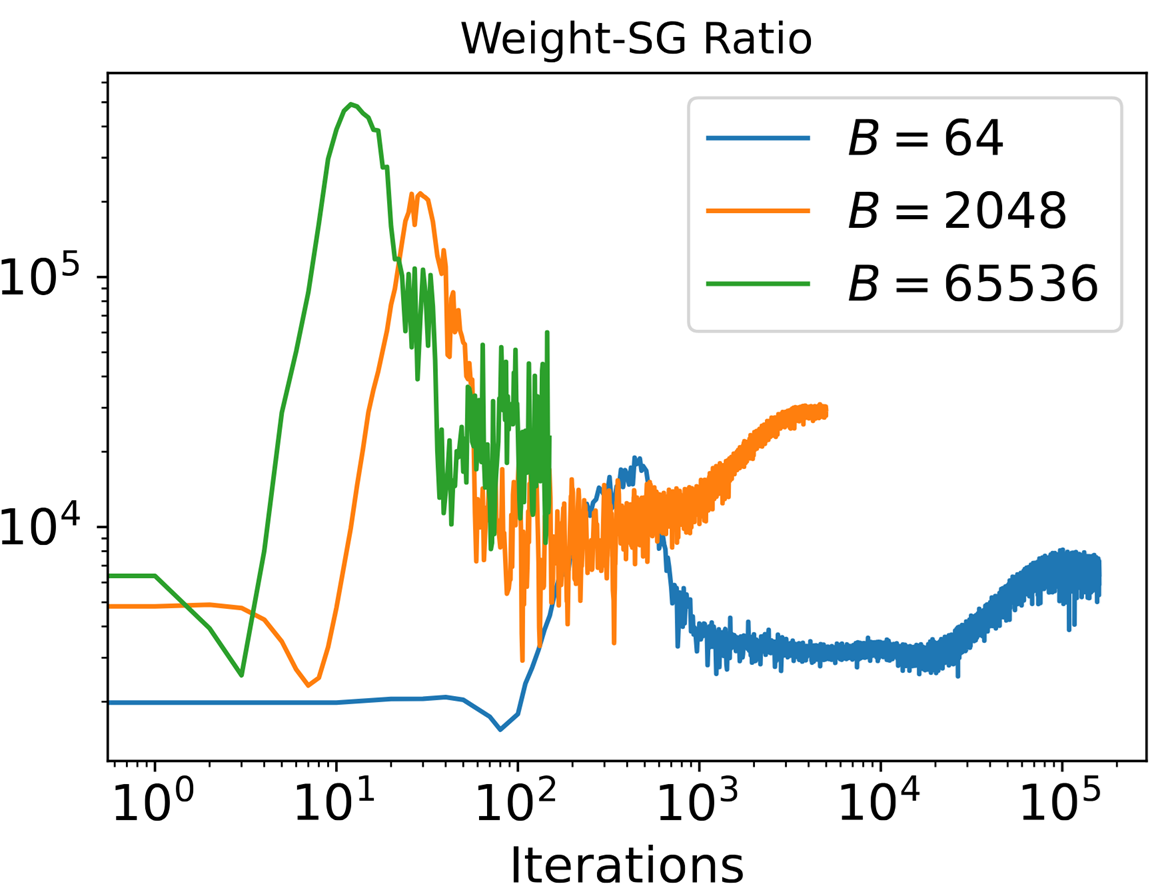}
        \caption*{(c)}
    \end{subfigure}
    \caption{(a) The dynamics of layer-wise weight to stochastic gradient (SG) norm ratio $\| \mathbf{w}^{(t)} \|_F / \| \nabla f({\bf w}^{(t)}; \xi^{(t)}) \|_F$ 
    during training. 
    (b) Layer-wise weight to momentum (of the stochastic gradient) norm ratio during training. As momentum is a smoothed version of the gradient (thus typically having smaller norm), this ratio tends to be even larger. Ratios of different layers are represented by different colors.
    (c) The weight-SG ratio under different batch-size settings. As batch size increases, the weight-SG ratio increases.
    }
    \label{fig:w_sg_ratio_training}
\end{figure}

To better understand the reason why 
the weight-SG norm ratio becomes so large during LLM pre-training, we conduct a more detailed theoretical analysis of the gradient matrix. 
In particular, we show that the gradient norm remains small when the batch size is large under LLM-specific training dynamics and token distribution. 
In the meantime, the weight norm is typically lower bounded (see Figure~\ref{fig:weight_norm}) in Appendix~\ref{apx:theory_justification}. Therefore,
the weight-SG ratio becomes large.

Let $\mathcal V=[V] := \{1, ..., V\}$ denote the vocabulary, and $\mathcal{D}$ denotes the distribution of the input-output data pair $(x, y)$. For some weight parameters $W$, we consider decomposing its mini-batch gradient $G$ into class-wise contributions $G_j$ such that 
\begin{equation}\label{eq:grad_decomposition_unified}
G_j
:=
\frac1B\sum_{b=1}^B
\bigl(p_j(\xi_b)-\mathbf 1\{y_b=j\}\bigr)U_j(\xi_b),
\qquad j\in[V],
\end{equation}
where $(\xi_b,y_b)$ are i.i.d. samples of forward and output information associated with the parameter block $W$ and data $(x_b, y_b)\sim\mathcal D$, and $p_j(\xi)$ is the model-predicted probability of class $j$ given input $\xi$.
$U_j(\xi)$ is the class-wise gradient feature associated with the weight parameter, which can be calculated in closed form as shown later. Further, define the marginal frequency of token $j$ and the corresponding conditional frequency on the input $\xi$ as 
\[
q_j:=\mathbb P_{(\xi, y)\sim\mathcal{D}}(y=j), \qquad
q_j(\xi):=\mathbb P_{(\cdot, y)\sim\mathcal{D}}(y=j\mid \xi).
\]

Let $\mathcal H$ and $\mathcal T$ be a partition of $[V]$. Then the full gradient norm admits the following upper bound.

\begin{theorem}\label{thm:full_grad_mainbody_unified}
Let $G$ be the mini-batch stochastic gradient of a parameter block satisfying \eqref{eq:grad_decomposition_unified}. Let two sets $\mathcal{H}$ and $\mathcal{T}$ be a partition of the tokens $\mathcal{V}\;\coloneqq\;\{1, ..., V\}$ such that $\mathcal{H}\cap\mathcal{T}=\varnothing$ and $\mathcal{H}\cup\mathcal{T}=\mathcal{V}$. Define $c_q\;\coloneqq\;\underset{j\in\mathcal{T}}{\max}~q_j$. Further, denote
\[
\varepsilon_j^2
:=
\mathbb E\!\left[
\|(p_j(\xi)-q_j(\xi))U_j(\xi)\|^2
\right],
\qquad
s_j^2
:=
\mathbb E\!\left[
\mathbf \|U_j(\xi)\|^2\mid y=j
\right].
\]
Assume that the full gradient $G$ satisfies
$\|G\|^2
\le
\kappa\sum_{j=1}^V\|G_j\|^2$
for some structural constant $\kappa\ge 1$.
Then the following holds:
\begin{equation}\label{eq:full_grad_upper_bound_universal_unified}
\begin{aligned}
\frac1\kappa\mathbb E\|G\|_F^2
&\le
\left(1+\frac1B\right)
\sum_{j\in\mathcal H}\varepsilon_j^2
+
\frac1B
\sum_{j\in\mathcal H}
\mathbb E\!\left[
q_j(\xi)(1-q_j(\xi))\|U_j(\xi)\|^2
\right]
\\
&\quad+
\left(\frac2B+2c_q\right)
\sum_{j\in\mathcal T}q_js_j^2
+
2\sum_{j\in\mathcal T}
\mathbb E\!\left[
p_j(\xi)^2\|U_j(\xi)\|^2
\right].
\end{aligned}
\end{equation}
\end{theorem}

Theorem~\ref{thm:full_grad_mainbody_unified} provides a unified upper bound for the gradient norm that holds for arbitrary choice of the partition and any degree of token heterogeneity. The bound separates the gradient norm $\|G\|_F$ into the following four contributing terms: {\it (i)} The class-wise population gradient contribution for head tokens, controlled by $\varepsilon_j$; {\it (ii)} The covariance contribution for head tokens, which scales as $O(1/B)$. {\it (iii)} The softmax background of tail tokens, controlled by $p_j(\xi)$. {\it (iv)} The label-hit contribution of tail tokens, controlled by $q_j$ (and uniformly by $c_q$). We note that $\kappa$ is found to be small in the pre-training setting; see Appendix \ref{apx:kappa_experiments} for details.

The main point of Theorem \ref{thm:full_grad_mainbody_unified} is to provide a mechanism for the large weight-SG norm ratio observed previously. In particular, a large batch size $B$ suppresses the covariance term, while a suitable partition can separate head tokens, for which $\varepsilon_j$ is small, from tail tokens, for which $p_j(\xi)$ and $c_q$ are sufficiently small. Thus under such partition, the right-hand side of
\eqref{eq:full_grad_upper_bound_universal_unified} can be small, leading to a small $\|G\|_F$ and hence a large weight-SG norm ratio.  We next explain why these conditions hold in LLM training:

\noindent $\bullet$ $\varepsilon_j$ is small for high-frequency tokens. High-frequency tokens are typically well-trained and exhibit lower
    per-token loss than low-frequency tokens \citep{chung2025exploiting}.  This suggests that $|p_j(\xi)-q_j(\xi)|$ is small, leading to a small $\varepsilon_j$.

 \noindent $\bullet$ $B$ is large. In large-scale LLM training, the effective batch size is measured by the number of processed tokens per update, which is often on the order of $10^5$--$10^7$.
 
 \noindent $\bullet$ $p_j(\xi)$ is small for low-frequency tokens. Due to limited exposure during training, rare tokens are more likely to
    remain under-trained \citep{bao2023token}. Therefore, their per-token loss (approximately $-\log p_j(\xi)$) can be high, indicating a smaller value of $p_j(\xi)$. 

 \noindent $\bullet$ $c_q$ is small for low-frequency tokens. Recall that $c_q\;\coloneqq\;\max_{j\in\mathcal{T}}~q_j$ is the maximal frequency for all tail tokens. It is known that tail tokens have very small marginal probabilities in natural language corpora \citep{piantadosi2014zipf}. 
 
More discussions about these terms can be found in Appendix \ref{apx:theory_justification}. 

Theorem \ref{thm:full_grad_mainbody_unified} gives a unified bound for all weight parameter blocks whose gradients admit class-wise decomposition \eqref{eq:grad_decomposition_unified}. For specific layers, this bound can be more concrete by identifying $\xi$ and $U_j(\xi)$. We illustrate this specialization in two representative cases, each with a different choice of $\kappa$.

\textbf{(i) Output layer.}
For the output-layer weight matrix $W\in\mathbb R^{V\times d}$, take $\xi=h$ and
$U_j(\xi)=h$ where $h\in\mathbb{R}^d$ is the output-layer hidden representation. Note that the logit and the model predictive distribution are
\[
z=Wh,
\qquad p(h) = \mathrm{softmax}(Wh) \in \mathbb{R}^V.
\]
Then
$
G
=
\frac1B\sum_{b=1}^B
(p(h_b)-e_{y_b})h_b^\top,
$
and the class-wise contribution $G_j$ is simply the $j$-th row ${\bf g}_j$ of $G$.
Therefore
$
\|G\|_F^2
=
\sum_{j=1}^V\|{\bf g}_j\|^2
=
\sum_{j=1}^V\|G_j\|^2,
$
which implies $\kappa=1$.  Consequently,
\[
\mathbb E\|G\|_F^2
\le
\left(1+\frac1B\right)
\sum_{j\in\mathcal H}\varepsilon_j^2
+
\frac1B\mathbb E\|h\|^2
+
\left(\frac2B+2c_q\right)
\sum_{j\in\mathcal T}q_js_j^2
+
2\sum_{j\in\mathcal T}
\mathbb E\!\left[
p_j(h)^2\|h\|^2
\right].
\]

\textbf{(ii) Intermediate layer.}
For an intermediate-layer weight matrix $W^{(l)}$ in layer $l$, take $\xi=h^{(l-1)}$ where $h^{(l-1)}$ is the hidden representation from the $(l-1)$-th layer. Due to the product-wise structure of the gradient, $U_j(\xi)$ can be represented by a rank-$1$ matrix. Then we can simply choose $\kappa=V$ to obtain the bound. We defer the detailed derivation for this case to Appendix \ref{apx:proof_hidden}.

The detailed proof of Theorem \ref{thm:full_grad_mainbody_unified} with the two cases can be found in Appendix~\ref{apx:proof} and Appendix~\ref{apx:proof_hidden}. Here we provide the proof outline.

By definition \eqref{eq:grad_decomposition_unified}, $G_j$ is the per-token-class gradient corresponding to token $j$. For any tokens $j\in\mathcal{H}$, we exploit a bias--variance decomposition:
\begin{equation}\label{eq:tokenwise_grad_batch_sec_mainbody}
    \mathbb{E}\bigl[\|G_j\|^2\bigr]
= \|\mathbb{E}[G_j]\|^2 + \mathrm{tr}\bigl(\mathrm{Cov}(G_j)\bigr).
\end{equation}
We note that $\mathbb{E}[G_j]$ represents the population class-wise gradient
contribution of token $j$.
Under this decomposition, we have
\begin{equation*}
    \|\mathbb{E}[G_j]\|^2\le\varepsilon_j^2, \qquad \mathrm{tr}\bigl(\mathrm{Cov}(G_j)\bigr)\le \frac{1}{B}\left(\varepsilon_j^2+\mathbb{E}\!\left[
q_j(\xi)(1-q_j(\xi))\|U_j(\xi)\|^2\right]\right), 
\end{equation*}
yielding the first two terms in the upper bound in \eqref{eq:full_grad_upper_bound_universal_unified}. This decomposition and the resulting bounds are used to bound tokens where the population gradient $\mathbb{E}[G_j]$ is relatively well estimated, so that the gradient norm is dominated by the covariance term. 

For any tokens $j\in\mathcal{T}$, we isolate the label hits $\{b:y_b=j\}$ via a background--hit decomposition:
\begin{equation}\label{eq:tokenwise_grad_rare_mainbody}
G_j
=\underbrace{\frac{1}{B}\sum_{b=1}^B p_j(\xi_b)\,U_j(\xi_b)}_{\text{``softmax background''}}-
\underbrace{\frac{1}{B}\sum_{b:y_b=j} U_j(\xi_b)}_{\text{``label hits''}}.
\end{equation}
The ``softmax background'' term involves the softmax output $p_j(\xi)$ of all samples in the batch, while the ``label hits'' term records the possible spikes involved when token $j$ is sampled. Under this decomposition, we have
\vspace{-.25cm}
\[
    \mathbb{E}\left\|\frac{1}{B}\sum_{b=1}^B p_j(\xi_b)\,U_j(\xi_b)\right\|^2
\le
\mathbb{E}\bigl[p_j(\xi)^2\|U_j(\xi)\|^2\bigr],\qquad
\mathbb{E}\left\|\frac{1}{B}\sum_{b:y_b=j} U_j(\xi_b)\right\|^2
\le\;
\frac{q_j}{B}s_j^2
+
q_j^2s_j^2,
\]

\vspace{-.25cm}

yielding the last two terms in the upper bound in \eqref{eq:full_grad_upper_bound_universal_unified}.
Decomposition \eqref{eq:tokenwise_grad_rare_mainbody} and the resulting bounds are used to characterize the regime where the token frequency is rare (so that $q_j$ is small), or where the token is under-trained (so that $p_j(\xi)$ is small).


\vspace{-5pt}

\section{Shrinking the SGD-Adam Gap in LLM Pre-Training}

\vspace{-5pt}

\subsection{Challenges in Shrinking the Gap: Irregular Stochastic Gradient}

\begin{figure}
    \centering
    \begin{subfigure}[t]{0.3\linewidth}
        \centering
        \includegraphics[width=\linewidth]{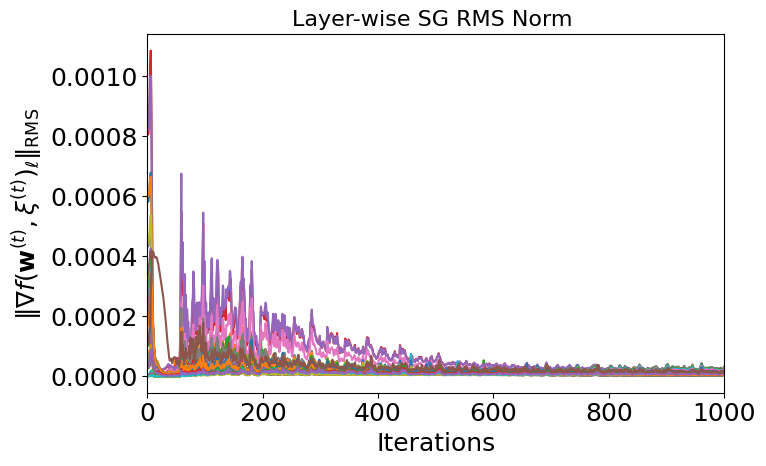}
        \caption*{(a)}
    \end{subfigure}
    \hfill
    \begin{subfigure}[t]{0.33\linewidth}
        \centering
        \includegraphics[width=\linewidth]{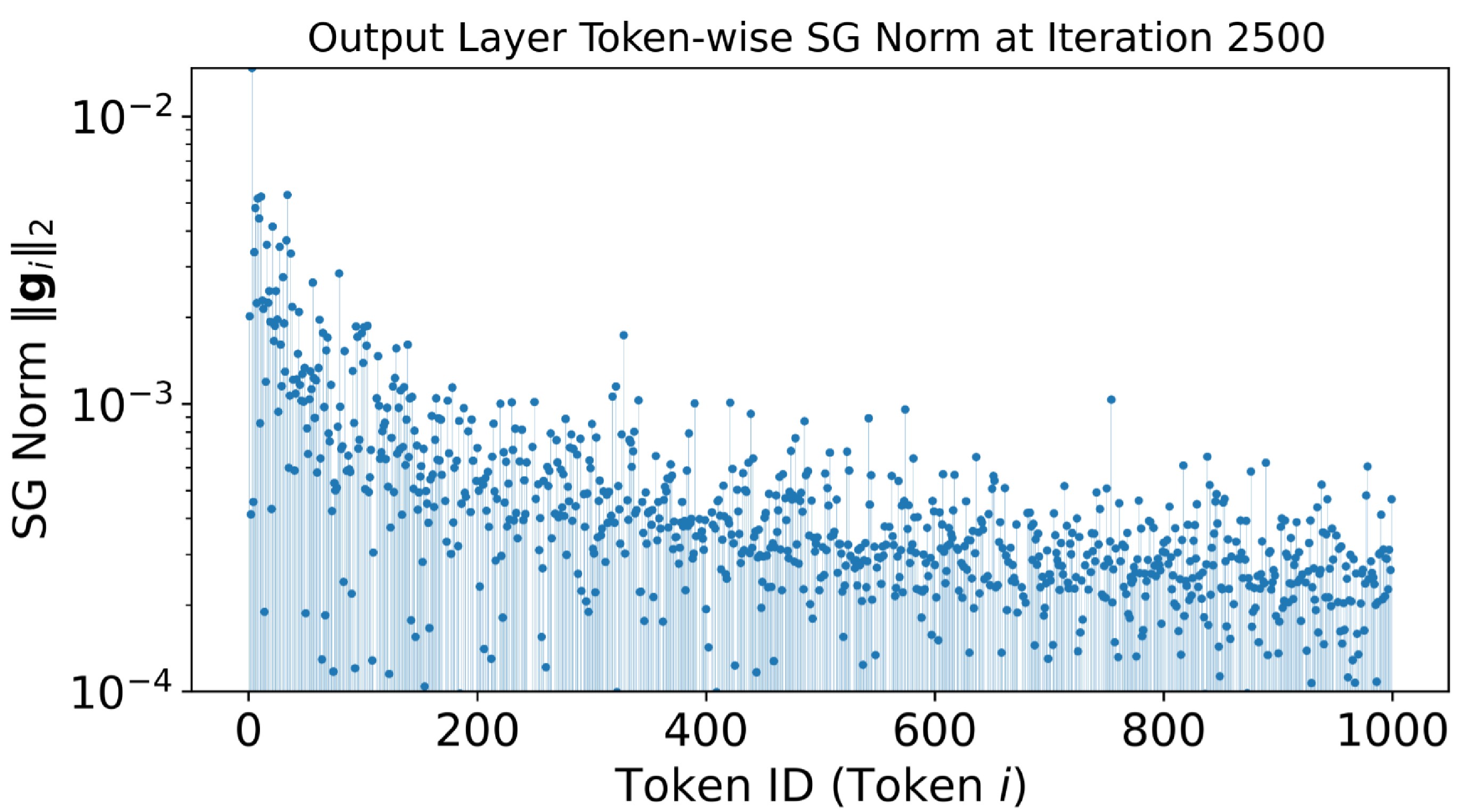 }
        \caption*{(b)}
    \end{subfigure}
    \hfill
    \begin{subfigure}[t]{0.32\linewidth}
        \centering
        \includegraphics[width=\linewidth]{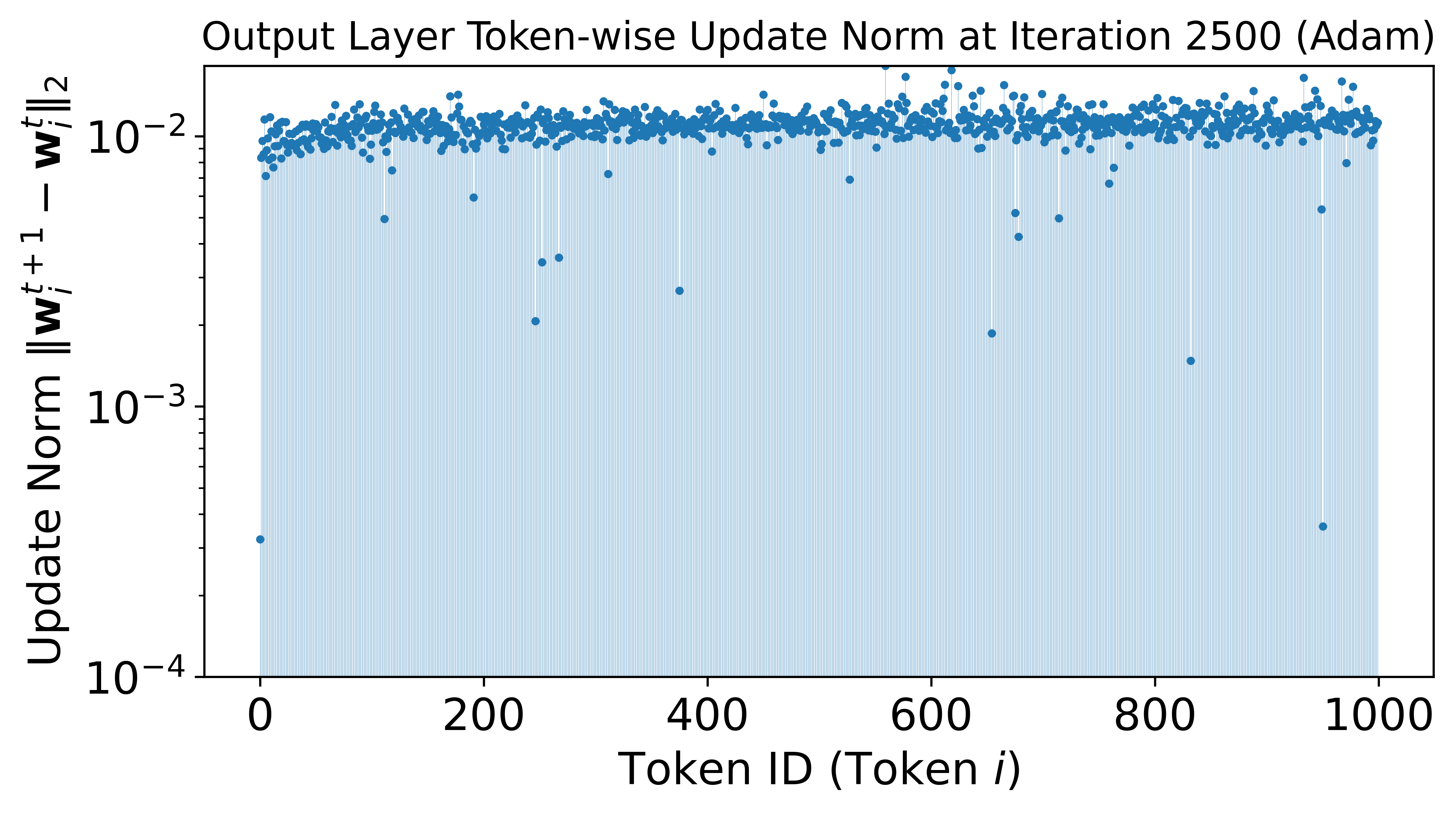}
        \caption*{(c)}
    \end{subfigure}
    \caption{ 
    (a): Layer-wise stochastic gradient norm $\| \nabla f({\bf w}^{(t)}; \xi^{(t)})_\ell \|_{\rm RMS}$ for all main layers $\ell = 1, ..., L$ of the $L$-layer model. 
    (b): Per-token-class stochastic gradient norm $\| {\bf g}_i \|_2$ for ${\bf g}_1, ..., {\bf g}_V$ the stochastic gradient of each token class output head (up to the first 1000 token IDs). Weight vectors corresponding to frequent token class (lower token ID) receive higher SG norm.
    (c): Per-token-class update norm $\| {\bf w}_i^{t+1} - {\bf w}_i^t \|_2$ of Adam for ${\bf w}_1, ..., {\bf w}_V$ the weight vector of each token class output head (up to the first 1000 token IDs). Adam effectively normalizes the heterogeneous token class SG.
    }
    \label{fig:grad_irreg}
    \vspace{-10pt}
\end{figure}

Now we have acknowledged that the effective learning rate of SGD is too small for pre-training. However, increasing the  learning rate of SGD further leads to divergence (e.g., see Figure 11 of \citet{zhang2024transformers}). In this section, we identify two instabilities that arise from the stochastic gradient that destabilize SGD under large learning rates.

\paragraph{(Type I) Layer-wise Stochastic Gradient Spikes.}
As illustrated in Figure~\ref{fig:grad_irreg} (a), we observe spikes in the form of unstable large stochastic gradient norm layer-wise, especially in the early stage of training. It is known that such spikes often result in a loss spike when the learning rate is large, leading to unstable training and divergence \citep{takase2023spike}.

 \vspace{-5pt}

\paragraph{(Type II) Output Layer SG Disparity.}
As we see from Figure~\ref{fig:grad_irreg} (b), high-frequency token classes (small token IDs) receive higher magnitude of stochastic gradient at the output layer classification head. Intuitively, this disparity makes the gradient concentrated on high-frequency tokens. This implies that even if the overall gradient norm is relatively small, the per-token-class gradients of the high-frequency tokens can still have a large magnitude, resulting in the training instability for these tokens. We notice that adaptive optimizers like Adam will normalize such heterogeneous token class stochastic gradient to stabilize training under class imbalanced samples, by observing the per-token-class weight update norm Figure \ref{fig:grad_irreg} (c). 

Below, we provide a formal theoretical understanding on why the output-layer gradient norm is skewed toward high-frequency tokens. Denote $[{\bf g}_1, ..., {\bf g}_V] \in \mathbb{R}^{d \times V}$ as the output layer gradient. We show that the per-token-class gradient norm $\mathbb{E}\|{\bf g}_j\|^2$ is approximately proportional to the token frequency in the early training stage, which can lead to substantial disparities across tokens.

\begin{theorem}\label{thm:grad_lower_and_ratio_mainbody}
Assume there exists a constant $0 < c \ll B$ such that 
$\mathbb{E}\bigl[p_j(h)^2\|h\|^2\bigr]
\;\le\;
\frac{c^2}{B^2}\,\mathbb{E}\|h\|^2.
$
In addition, define $m_j\coloneqq \mathbb{E}[h\mid y=j]$, $s_j^2\coloneqq \mathbb{E}[\|h\|^2\mid y=j]$ and assume that
$0<s_j^2<\infty$ for the tokens under consideration.
Then,
\begin{equation*}
\mathbb{E}\|{\bf g}_j\|^2
\;\ge\;
\frac{1}{2B}\,q_j\,s_j^2
\;-\;
\frac{c^2}{B^2}\,\mathbb{E}\|h\|^2.
\end{equation*}
\vspace{-.5cm}

Consequently, for any two tokens $i, j$,
\begin{equation}\label{eq:grad_ratio_bound_mainbody}
\frac{\mathbb{E}\|{\bf g}_i\|^2}{\mathbb{E}\|{\bf g}_j\|^2}
\;\ge\;
\frac{\frac{1}{2B}q_i s_i^2-\frac{c^2}{B^2}\mathbb{E}\|h\|^2}
{\frac{2q_j}{B}\,s_j^2
+2q_j^2\|m_j\|^2
+\frac{2c^2}{B^2}\,\mathbb{E}\|h\|^2}.
\end{equation}
\vspace{-.5cm}

\end{theorem}

Theorem~\ref{thm:grad_lower_and_ratio_mainbody} provides the lower bound on the per-token-class gradient in the early training stage, showing that it scales with the token frequency. In particular, when two tokens differ significantly in frequency, i.e., $q_i/q_j$ is large, we obtain 
\begin{equation*}
    \frac{\mathbb{E}\|{\bf g}_i\|^2}{\mathbb{E}\|{\bf g}_j\|^2} = \Omega\left(\frac{q_i}{q_j}\right)
\end{equation*}
for small enough $c^2/B^2$.
This explains the empirical observation that high-frequency tokens tend to exhibit much larger per-token-class gradient norms in the output layer. Furthermore, for infrequent tokens with sufficiently small $q_j\|m_j\|$, the lower bound of the ratio in \eqref{eq:grad_ratio_bound_mainbody} increases with $B$ over a large range of batch sizes, given other terms fixed. Empirically, we also observe greater disparity between per-token-class gradient norms at larger batch sizes; see Figure \ref{fig:token_class_variation} in Appendix \ref{apx:figures}.

\vspace{-5pt}

\subsection{Large Learning Rate SGD Enabled via Clipping Shrinks the Gap}

\label{sec:clipping}

With the knowledge about the Type-I and Type-II irregularity of the stochastic gradients that can potentially cause the SGD instability, we use two simple forms of clipping, being layer-wise RMS norm clipping and output layer per-token-class clipping, to mitigate the issues and enable  SGD with \underline{L}arge \underline{L}earning rate, which we refer to as \ourname{}.

More specifically, define the clipping operator as $\operatorname{clip}_{r,\|\cdot\|}(\bf x)
\coloneqq
\min\left\{1,\frac{r}{\|{\bf x}\|}\right\}\bf x$. For layer weights ${\bf w}_1, ..., {\bf w}_{L-1}$ and their momentums ${\bf m}_1, ..., {\bf m}_{L-1}$, we apply layer-wise RMS norm clipping:
\[
{\bf w}^{(t+1)}_{\ell} = {\bf w}^{(t)}_{\ell} - \operatorname{clip}_{\tau,\|\cdot\|_{\rm RMS}}
\!\left(\eta_t {\bf m}_{\ell}^{(t)}\right), \qquad \text{for} ~ \ell = 1,..., L-1.
\]
For the output layer classifier head class weights ${\bf w}_{L, 1}, ..., {\bf w}_{L, V}$ and their momentums ${\bf m}_{L, 1}, ..., {\bf m}_{L, V}$, we additionally apply output layer per-token-class clipping:
\[
{\bf w}^{(t+1)}_{L} = {\bf w}^{(t)}_{L} - \operatorname{clip}_{\tau,\|\cdot\|_{\rm RMS}}
\!\left(\eta_t \widetilde{\bf m}_{L}^{(t)}\right), \quad \widetilde{\bf m}^{(t)}_{L,i} = \operatorname{clip}_{\delta,\|\cdot\|_2}
\!\left(m_{L,i}^{(t)}\right),
        \qquad \text{for} ~ i=1,..., V.
\]
Here $\tau > 0$ is the layer-wise clipping threshold and $\delta > 0$ is the per-token-class clipping threshold.

\begin{figure}[tp]
    \centering
    \begin{subfigure}{0.85\linewidth}
        \centering
        \includegraphics[width=\linewidth]{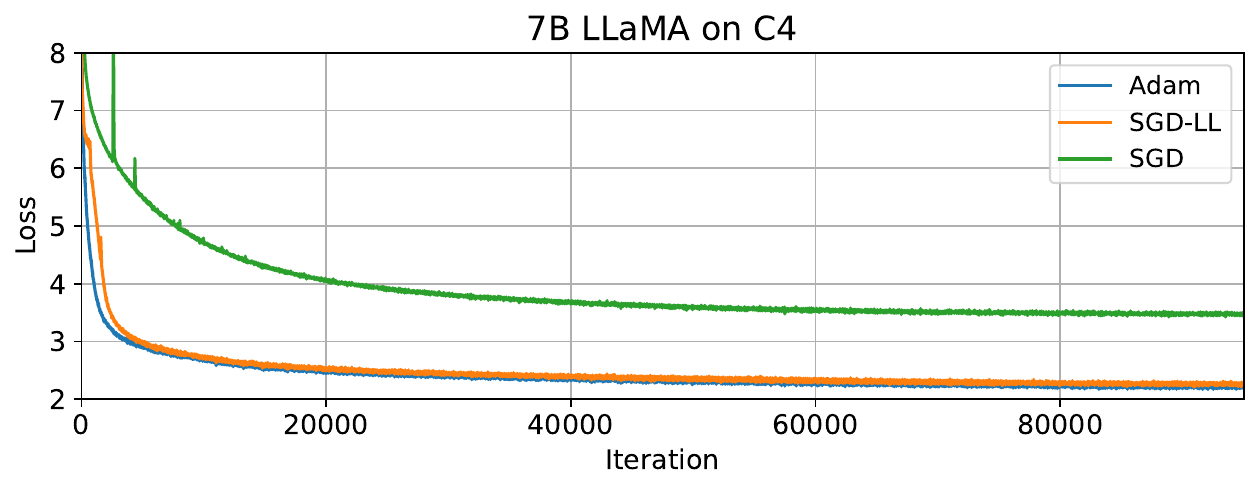}
    \end{subfigure}
    \vspace{-0.2cm}
    \caption{Training loss curves for the 7B model pre-trained on 100B C4 tokens. The loss trajectory of \ourname{} closely matches that of Adam, whereas SGD maintains substantially higher loss.}
    \label{fig:loss_curves_7b}
\end{figure}

\begin{table}[tp]
\centering
\renewcommand{\arraystretch}{1.3}
\setlength{\tabcolsep}{3pt}
\resizebox{1\textwidth}{!}{%
\begin{tabular}{l|ccc|ccc|ccc|ccc}
\toprule
 & \multicolumn{3}{c}{\textbf{130M} (2.6B tokens)} 
 & \multicolumn{3}{c}{\textbf{350M} (7.8B tokens)} 
 & \multicolumn{3}{c}{\textbf{1B} (20B tokens)}
 & \multicolumn{3}{c}{\textbf{7B} (100B tokens)}  \\
\midrule
 & Val. Loss & PPL & LR 
 & Val. Loss & PPL & LR 
 & Val. Loss & PPL & LR  
 & Val. Loss & PPL & LR \\
\midrule
\textbf{Adam}
& 3.11 & 22.46 & $3 \times 10^{-3}$ 
& 2.78 & 16.09 & $2 \times 10^{-3}$ 
& 2.51 & 12.32 & $2 \times 10^{-3}$ 
& 2.21 & 9.15  & $1 \times 10^{-3}$ \\
\midrule
\textbf{SGD}
& 5.07 & 158.90 & 0.5
& 4.14 & 63.22 & 0.5 
& 3.95 & 51.87 & 0.5
& 3.48 & 32.32 & 0.25 \\
\midrule
\textbf{\ourname{}}
& 3.24 & 25.70 & 100 
& 2.85 & 17.35 & 200 
& 2.60 & 13.47 & 300 
& 2.26 & 9.63  & 300 \\
\bottomrule
\end{tabular}
}
\vspace{5pt}
\caption{
Comparison on training various model sizes on C4. We report validation loss, perplexity (PPL) and base learning rate (LR). For the 130M and 350M models we use a 0.5M token batch size while for the 1B and 7B we increase it to 1M tokens.  
}
\label{tab:main_results}
\vspace{-15pt}
\end{table}

Notice that the layer-wise RMS norm clipping ensures that $\| {\bf w}_\ell^{(t+1)} - {\bf w}_{\ell}^{(t)} \|_{\rm RMS} \leq \tau$, ensuring no layers will be impacted by the Type-I gradient spikes. Moreover, it only clips the update during gradient spikes. Separately, the output layer clipping avoids the classifier head of frequent classes to diverge when receiving disproportionately large Type-II heterogeneous stochastic gradient. 
Finally, we remark that we find both types of clipping to be necessary to achieve the larger admissible learning rate for SGD (see Appendix \ref{apx:abl_clipping} for ablation).

We present our main numerical results in Table \ref{tab:main_results}. We observe that \ourname{}'s performance is close to Adam with a small margin, significantly improving from SGD. For instance, in training the 7B LLaMA model, the relative performance gap in validation loss reduces from 57.47\% to just 2.26\%.

Additionally, the training loss curves of each algorithm are illustrated Figure \ref{fig:loss_curves_7b} for the 7B model and in Appendix \ref{sec:loss_curves} for the rest, where we observe that the training trajectory of \ourname{} closely follows that of Adam, while the SGD ones show a considerable gap. We also include further experiments in the Appendix, including  pre-training GPT-2 models on FineWeb dataset \ref{apx:gpt_experiments}, multi-seed runs \ref{apx:seeds}, overtraining experiments \ref{apx:overtraining}, and downstream results with the 7B model \ref{apx:downstream}. 

Finally we note that the two clipping thresholds are chosen sufficiently large so that the intervention specifically targets the gradient irregularities, leaving the majority of the SGD's update unchanged; see Figure \ref{fig:sgd_cc_update_norm} in Appendix for illustration and \ref{apx:clip_freq} and \ref{apx:cosine_sim} for a quantitative analysis of the clipping frequency and the update cosine similarity. Together with the performance results, this supports the hypothesis that the lack of large learning rates is a key limiting factor for SGD in pre-training.

\vspace{-5pt}

\section{Conclusions}

Overall, our work provides both empirical and theoretical evidence revealing key underlying reasons behind the SGD's poor performance on LLM pre-training. We observe that the smaller update magnitudes of SGD, induced by its much lower effective learning rate relative to Adam, lead to slower training progress. Further, we theoretically connect the need for larger effective learning rates to the small gradient norms during pre-training. Finally, we identify two forms of gradient irregularities causing SGD to become unstable at the large learning rate regime; after appropriate handling we demonstrate that this regime closes most of the gap with Adam.

A limitation of this work is that the current simple remedy leaves a small residual performance gap to Adam. However, we emphasize that the focus of this study is not to propose new optimization algorithms, but rather on gaining understanding about important pre-training phenomena that shed more light to the important open question about the SGD and Adam performance gap in pre-training.

\begin{ack}
A.\ Glentis, D.\ Li, C.Y.\ Yau and M. Hong are partly supported by NSF Awards ECCS-2426064 and EPCN-2311007.

This work used NVIDIA A100 GPUs at NCSA Delta through allocation CIS261088 from the Advanced Cyberinfrastructure Coordination Ecosystem: Services \& Support (ACCESS) program \citep{boerner2023access}, which is supported by U.S. National Science Foundation grants \#2138259, \#2138286, \#2138307, \#2137603, and \#2138296.

The authors further acknowledge the Minnesota Supercomputing Institute (MSI) at the University of Minnesota for providing resources that contributed to the research results reported within this paper.

\end{ack}

\bibliographystyle{plainnat}
\bibliography{reference.bib}


\newpage

\appendix

\section{Additional Experimental Figures}\label{apx:figures}

\begin{figure}[H]
    \centering
    \begin{subfigure}[t]{0.37\linewidth}
        \centering
        \includegraphics[width=\linewidth]{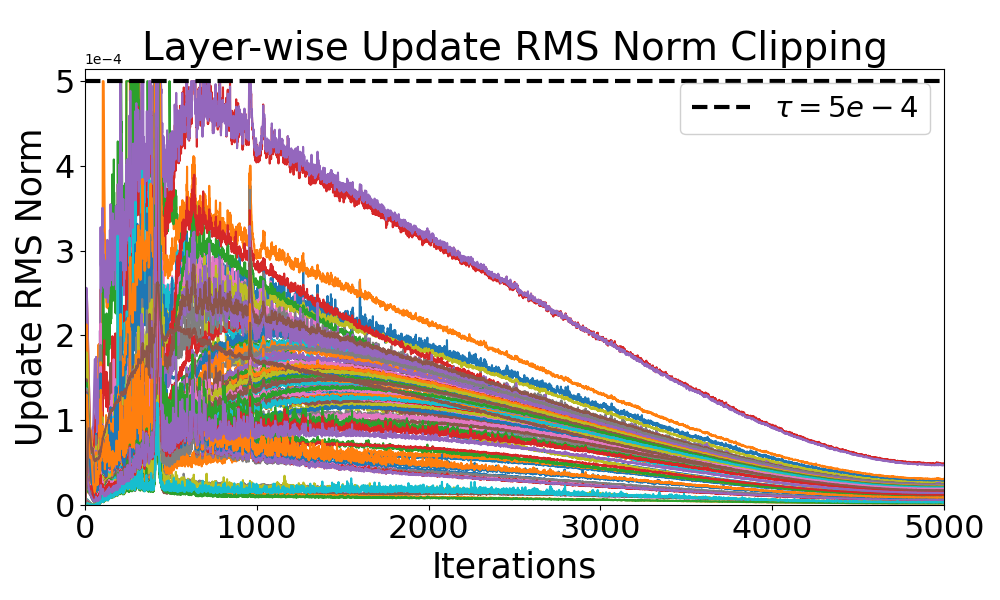}
    \end{subfigure}
    \hspace{10pt}
    \begin{subfigure}[t]{0.37\linewidth}
        \centering
        \includegraphics[width=\linewidth]{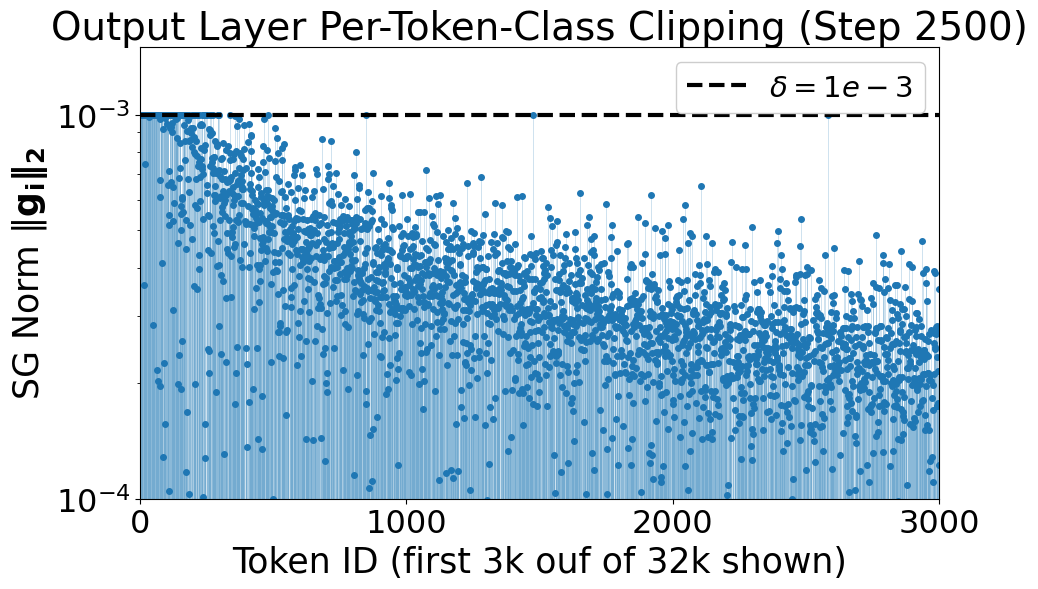}
    \end{subfigure}
    \vspace{-5pt}
    \caption{Illustration of the two forms of gradient clipping in \ourname{} while training LLaMA 130M, ``treating'' Type-I (Left) and Type-II (Right) gradient irregularities. Notice that clipping is targeted to irregularities and the vast majority of gradients remain unchanged. In the left figure update norms of different layers are represented by different colors.
    }
    \label{fig:sgd_cc_update_norm}
\end{figure}

\begin{figure}[H]
    \centering
    \includegraphics[width=0.5\linewidth]{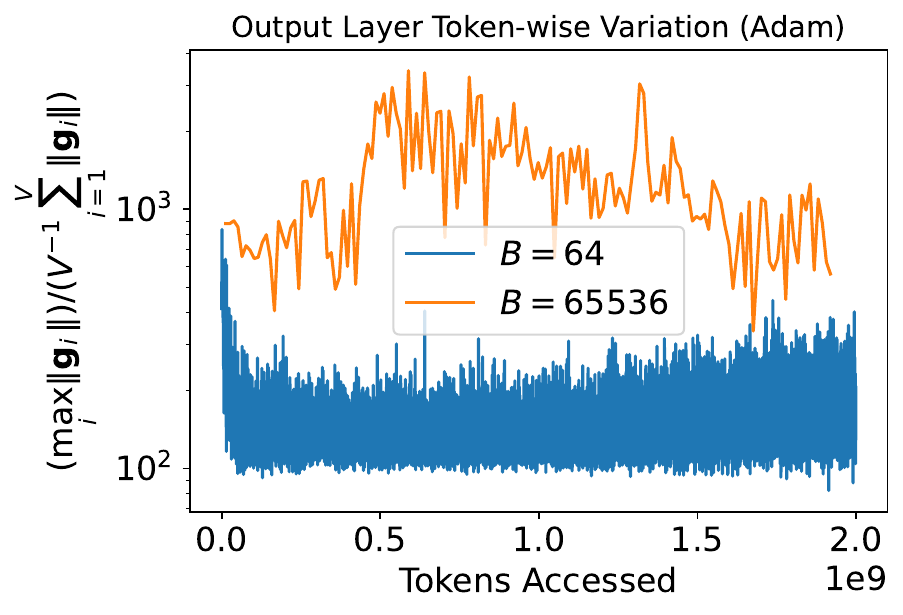}
    \caption{Measuring the token class gradient imbalance by the ratio between maximum token class gradient norm and average token-class gradient norm: $(\max_{i = 1,..., V} \| {\bf g}_i \|_{2}) / (V^{-1} \sum_{i=1}^V \| {\bf g}_i \|_2)$.}
    \label{fig:token_class_variation}
\vspace{-.5cm}
\end{figure}

\section{Training Loss Curves}

\label{sec:loss_curves}

We plot the training loss curves from the 130M-1B pre-training experiments below. We observe that the \ourname{} trajectory closely follows that of Adam, while SGD shows a significant gap.

\begin{figure}[H]
    \centering
    \begin{subfigure}{0.33\linewidth}
        \centering
        \includegraphics[width=\linewidth]{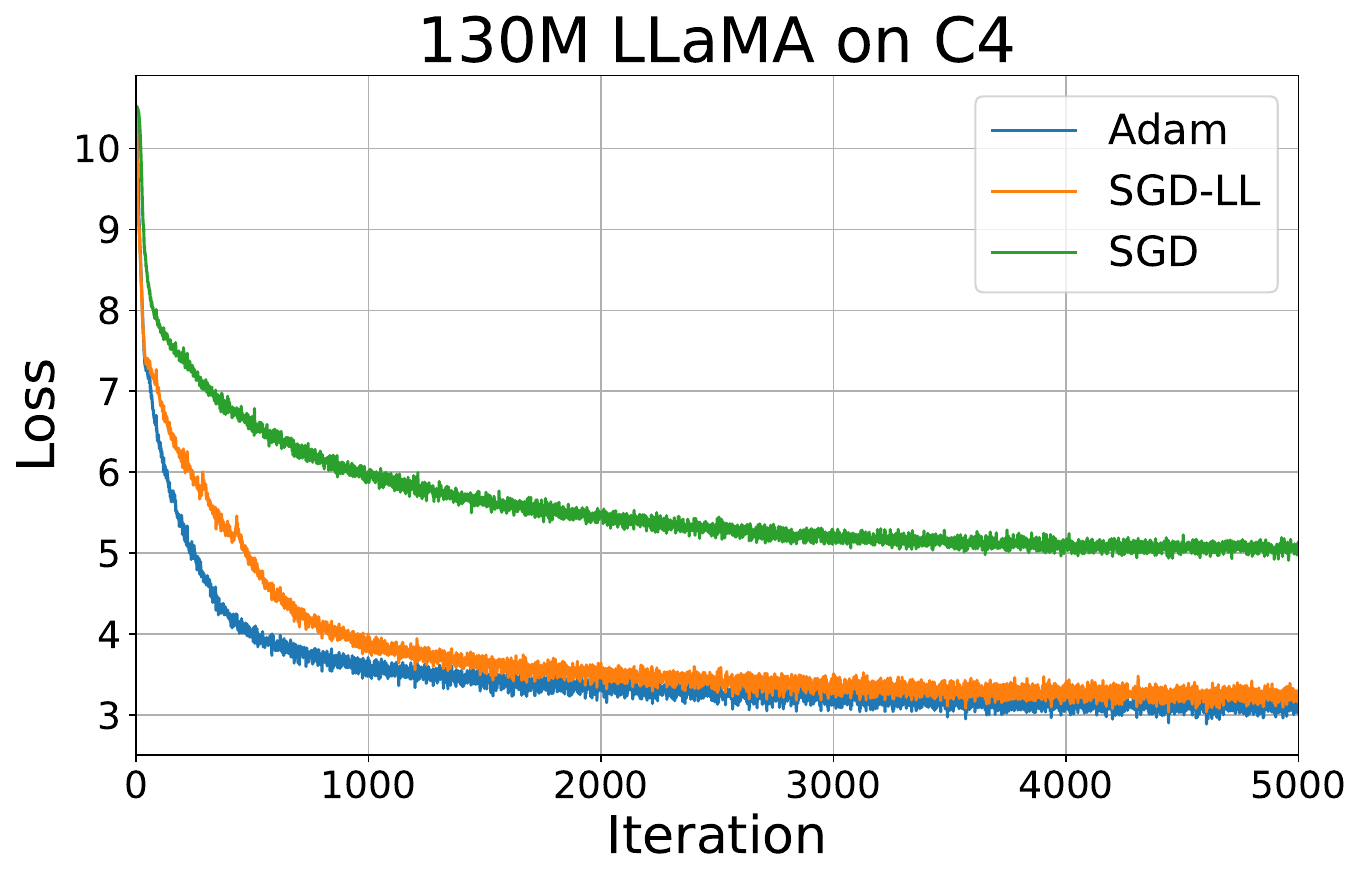}
    \end{subfigure}
    \begin{subfigure}{0.32\linewidth}
        \centering
        \includegraphics[width=\linewidth]{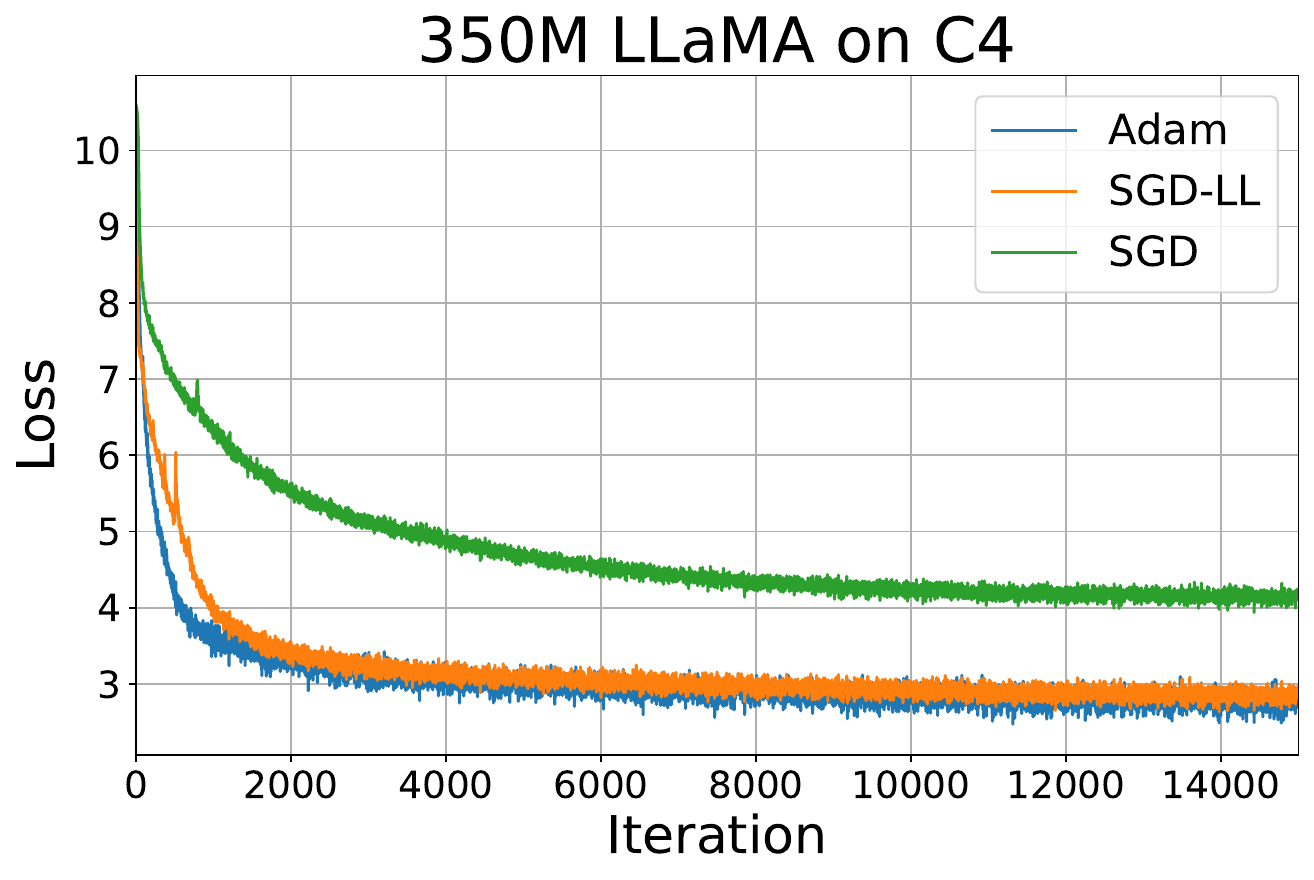}
    \end{subfigure}
    \begin{subfigure}{0.32\linewidth}
        \centering
        \includegraphics[width=\linewidth]{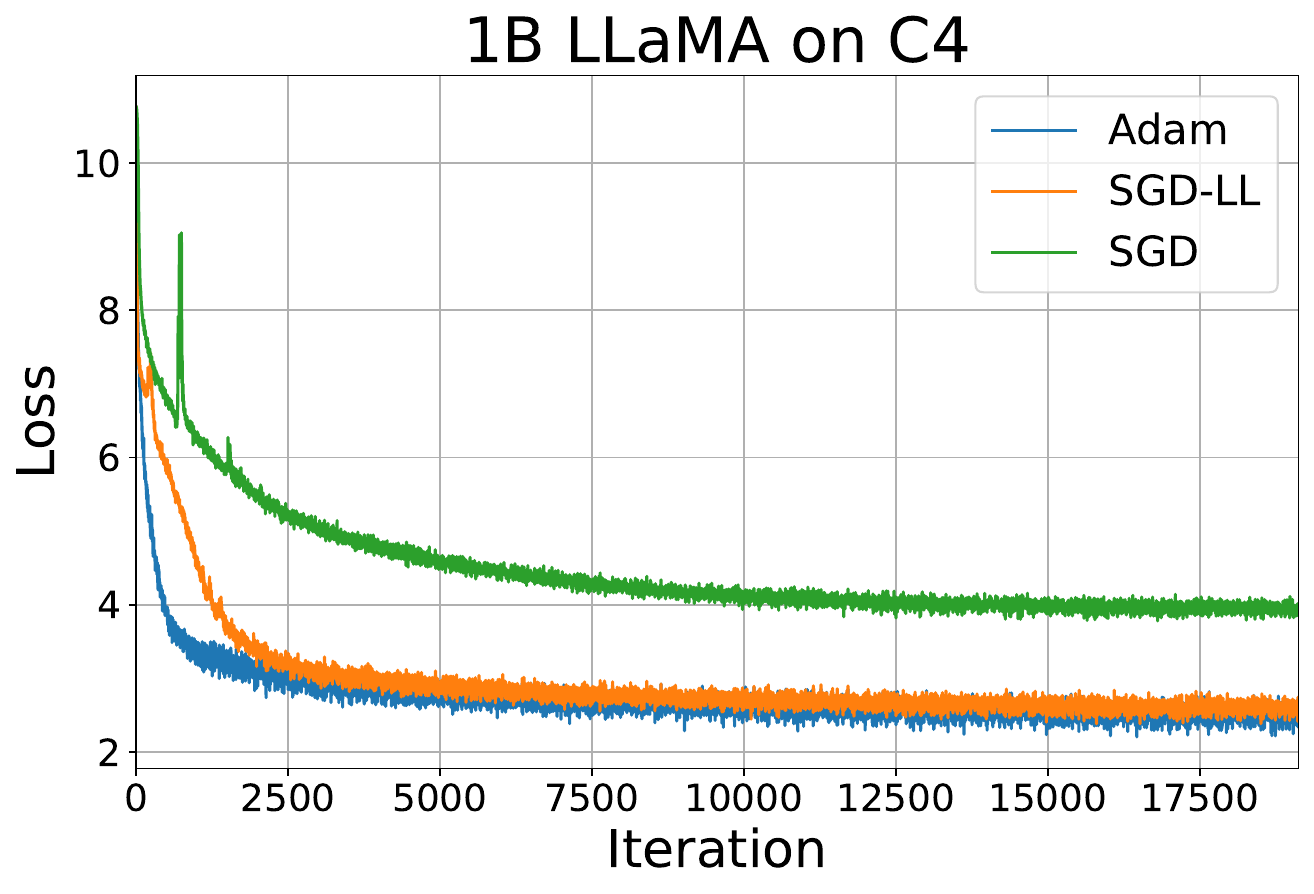}
    \end{subfigure}
    \vspace{-0.2cm}
    \caption{Training loss for model sizes 130M-1B.}
    \label{fig:loss_curves}
    \vspace{-10pt}
\end{figure}

\section{Experimental Settings}
\label{apx:experimental_settings}

 For our main analysis we pre-train the LLaMA \cite{touvron2023llama} decoder-only transformer model with about 130M trainable parameters on the C4 (Colossal Clean Crawled Corpus) dataset \cite{raffel2020exploring}. For tokenization, we use the T5 tokenizer \cite{raffel2020exploring} with a vocabulary of 32,000 tokens. We use a data-parallel setup with batch size 2048 (unless stated otherwise in the text) and sequence length 256, giving a token batch size of 524,288 tokens. We train for a total of 5,000 steps, such that the model sees roughly the Chinchilla compute-optimal \cite{hoffmann2022training} number of tokens (2.6B, i.e.,\ 20 tokens per model parameter). We train using Pytorch's Automatic Mixed Precision (AMP, BF16/FP32) and use the typical cosine learning rate scheduler with linear warm-up for the first 10\% of the iterations.
 
 For our large scale experiments, we train the 1B LLaMA model with 256 sequence length and 4096 batch size and the 7B LLaMA model with 512 sequence length (following \cite{semenov2025benchmarking}) and 2048 batch size. Both cases result to a token batch size of 1,048,576 tokens. We choose to increase the sequence length for the 7B model in order to be able to perform zero-shot downstream testing (see Section \ref{apx:downstream}) which tends to need larger sequence lengths depending on the task. The 1B and 7B models are trained on 20B and 100B C4 tokens, respectively.
 
 We use the loss and perplexity (defined as the exponential of the loss) of the final model checkpoint, measured on the typical C4 validation split, as our evaluation metrics. Learning rates are tuned based on best validation perplexity. 
 
 For \ourname{}, for the 130M and 350M models we choose $\tau = 5 \times 10^{-4}$. For the 1B model and 7B, we choose $\tau = 2\times 10^{-4}$ and  $\tau = 1\times 10^{-4}$, respectively. We adopt $\delta = 10^{-3}$ for all models. Furthermore, 
 following typical choices we use $\beta=0.9$ for the first-order momentum of both SGD and Adam and $\beta_2=0.95$ for the second order momentum of Adam \cite{brown2020language, touvron2023llama, touvron2023llama2}. We set weight decay to zero for all our main experiments in order to isolate the effect of the base optimizers, but include an ablation on weight decay in Section \ref{apx:adam_ablations}. 
 
 Finally, we note that we trained the  130M and 350M models on 4 NVIDIA A100-40GB GPUs and the larger 1B model on 4 NVIDIA H100-96GB GPUs. The time it takes to pre-train the 130M,  350M and 1B models in our setup is about 2, 8 and 30 hours, respectively.

\section{Additional Experiments}

\subsection{Ablations on necessity of both clipping types}
\label{apx:abl_clipping}

In this section we show that using only one of the two clipping types is inadequate to enable large learning rate SGD to close most of the gap.

\begin{figure}[H]
    \centering
\includegraphics[width=0.60\linewidth]{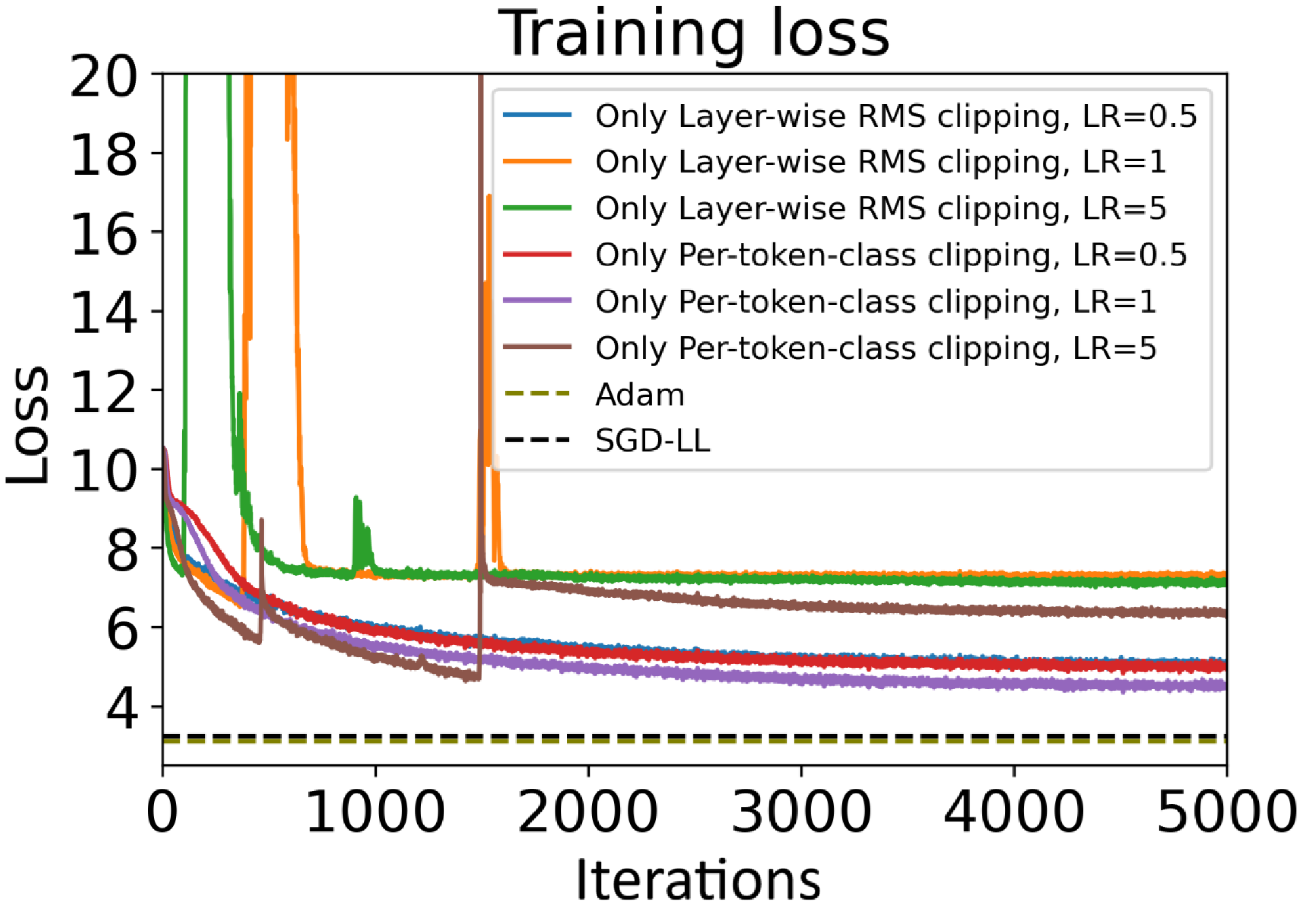}
\vspace{-.1cm}
    \caption{Training loss figure of 130M LLaMA on C4 pre-trained using   SGD with either Layer-wise update RMS Norm clipping or output layer per-token-class clipping. We observe that in either case increasing the learning rate higher than 1 leads to loss spikes and divergence. }
    \label{fig:abl_clipping_loss}
\vspace{-.5cm}
\end{figure}

\vspace{5pt}

\begin{table}[H]
\label{tab:clip_ablations}
\centering
\small
\resizebox{0.65\textwidth}{!}{%
\begin{tabular}{ccc}
\toprule
          & Learning Rate  & Val. Loss    \\
\midrule
Adam 	                    & 3e-3 	& 3.11 \\
SGD + Layer-wise-RMS-clip 	& 0.5 	& 5.06 \\
SGD + Layer-wise-RMS-clip 	& 1 	& 7.29 (diverged)  \\
SGD + Layer-wise-RMS-clip 	& 5 	& 7.12 (diverged)  \\
SGD + Per-token-class-clip 	& 0.5 	& 5.00  \\
SGD + Per-token-class-clip 	& 1 	& 4.52  \\
SGD + Per-token-class-clip 	& 5 	& 6.35 (diverged)  \\
SGD-LL (both clippings) 	& 100 	& 3.24  \\

\bottomrule
\end{tabular}
}
\vspace{5pt}
\caption{Final validation losses for each strategy corresponding to the trainings shown in Figure \ref{fig:abl_clipping_loss}.
}
\end{table}

\subsection{Ablations on clipping frequency}
\label{apx:clip_freq}

In this ablation we further analyze the impact of clipping in the final update. It is shown in Figure \ref{fig:sgd_cc_update_norm} that the applied clippings only affects a small number of layers and output classes. We further quantify this amount in the Figure below for the 130M and 1B models. We can see that out of 32,000 token classes, less than 100 classes are clipped during $98\%$ of the whole training period in both top figures. Similarly,  both bottom figures show that most of the time no layers are clipped via the layerwise clipping.

\begin{figure}[H]
    \centering
    \begin{subfigure}[t]{0.385\linewidth}
        \centering
        \includegraphics[width=\linewidth]{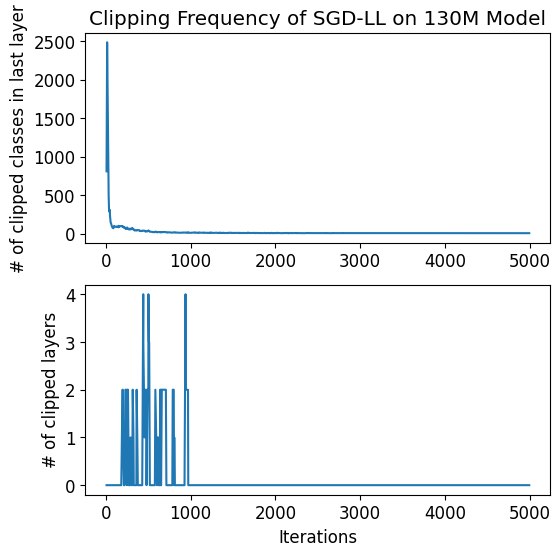}
    \end{subfigure}
    \hspace{10pt}
    \begin{subfigure}[t]{0.40\linewidth}
        \centering
        \includegraphics[width=\linewidth]{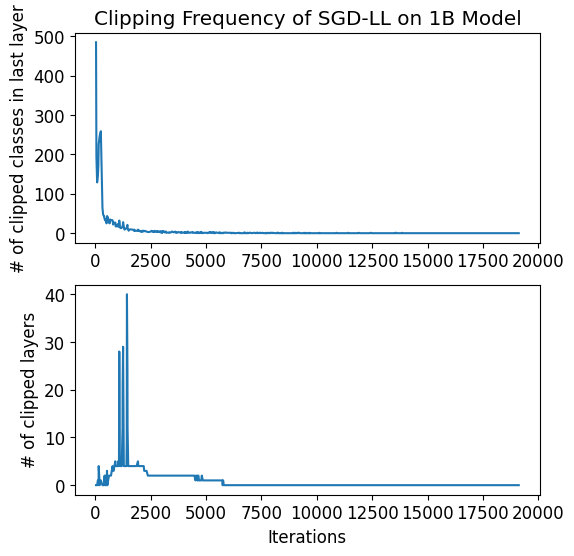}
    \end{subfigure}
    \vspace{-5pt}
    \caption{The number of token classes clipped in the last layer (top) and the number of clipped intermediate layers (bottom) for the 130M (left) and 1B (right) models during training.  
    }
    \label{fig:clip_freq}
\end{figure}

\subsection{Ablation on update cosine similarity after clipping}
\label{apx:cosine_sim}
 
We show that the clipping applied to the update does not substantially change the SGD's update direction during training. Specifically, each training step we computed the cosine similarity between the full vector update before and after SGD-LL's clipping and present the result in the Figure below.  We observe that, after the unstable early training phase, the cosine similarity peaks to high values close to 1.

\begin{figure}[H]
    \centering
    \begin{subfigure}[t]{0.75\linewidth}
        \centering
        \includegraphics[width=\linewidth]{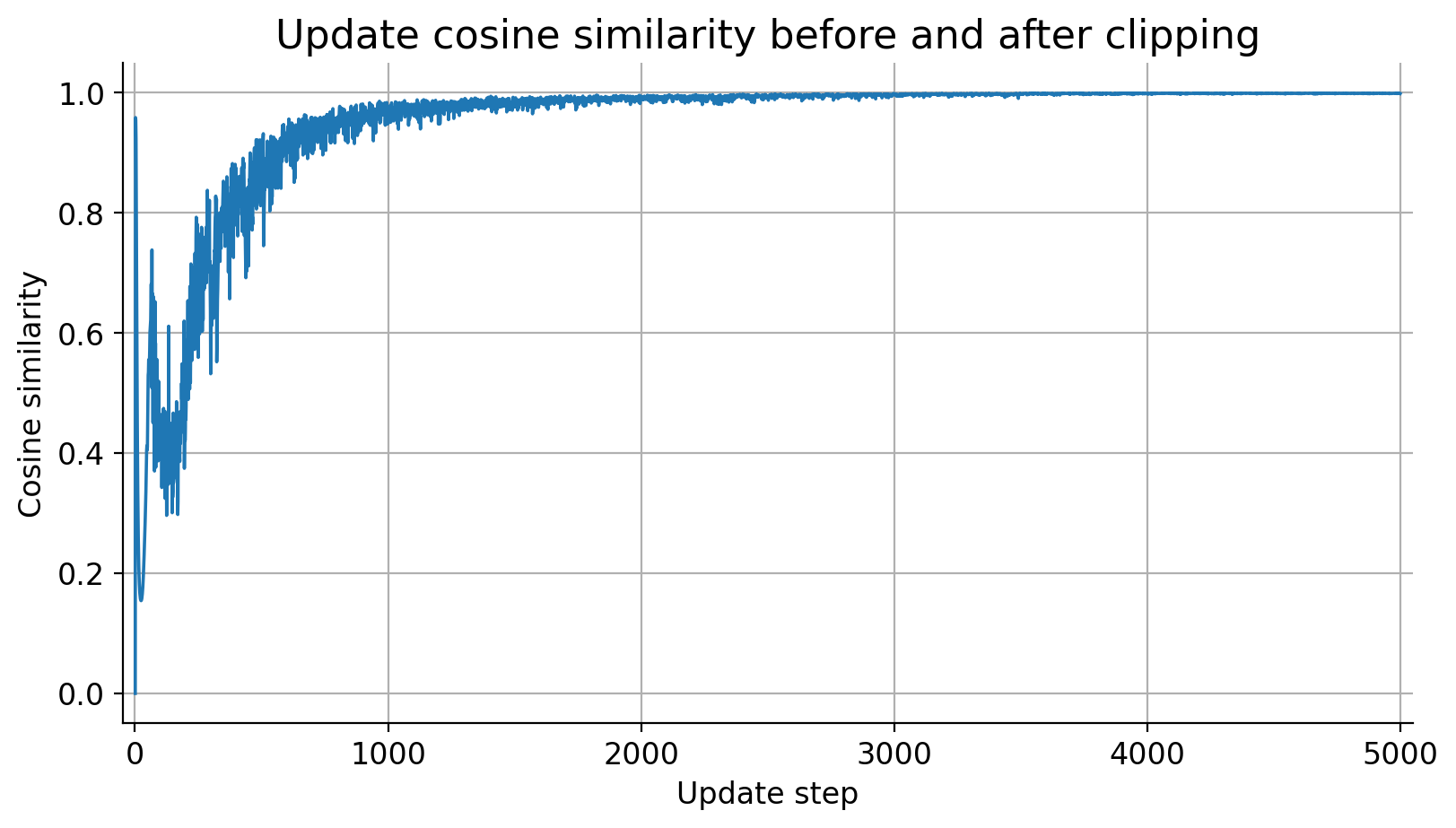}
    \end{subfigure}
    \vspace{-5pt}
    \caption{Cosine similarity between the full vector update before and after SGD-LL's clipping during pre-training the 130M LLaMA on C4. 
    }
    \label{fig:cosine_sim}
\end{figure}

\subsection{Further tuning of Adam hyperparameters}

\label{apx:adam_ablations}

Below we show that further tuning of Adam hyperparameters ($\beta_2$ and weight decay) has very limited impact on its performance and does not substantially change the observed gap. Such small differences are mainly relevant for the design  of new state-of-the-art optimizers which is not the focus of our work.  

\textbf{Ablation on $\beta_2$ of Adam:}

\begin{table}[H]
\label{tab:130m_beta_2_abl}
\centering
\small
\resizebox{0.5\textwidth}{!}{%
\begin{tabular}{cccccc}
\toprule
$\beta_2$   & 0.9  & 0.95  & 0.98  & 0.99   & 0.999  \\
\midrule
Val. Loss            &  3.115    & 3.112 & 3.114 & 3.114  &  3.145  \\
\bottomrule
\end{tabular}
}
\vspace{5pt}
\caption{Tuning the $\beta_2$ hyperparameter of Adam on the 130M model with 0.5M token batch size and learning rate 3e-3. Using $\beta_2=0.95$ gives the best validation loss (which is the value used in the main experiments). 
}
\end{table}

\textbf{Ablation on weight decay:}

\begin{table}[H]
\vspace{5pt}
\label{tab:130m_adamw_abl}
\centering
\small
\resizebox{0.5\textwidth}{!}{%
\begin{tabular}{ccccccc}
\toprule
$wd$             & 0.0      & 0.01  &  0.05 & 0.1     & 0.2   & 0.5  \\
\midrule
Val. Loss             &  3.112    & 3.110 & 3.103 & 3.100 & 3.096  & 3.106  \\
\bottomrule
\end{tabular}
}
\vspace{5pt}
\caption{Tuning the weight decay ($wd$) hyperparameter of Adam(W) on the 130M model with 0.5M token batch size and learning rate 3e-3. Note that when $wd=0$, AdamW practically reduces to Adam. We observe that the use of weight decay can give a small performance improvement, however its effect is minor (only about 0.5\% improvement in loss).}
\end{table}

\subsection{Experiments with GPT-2 architecture}

\label{apx:gpt_experiments}

\begin{table}[h]
\centering
\resizebox{0.55\textwidth}{!}{%
\begin{tabular}{lcc}
\toprule
   & NanoGPT (125M) & GPT-medium (355M) \\
\midrule
Adam        & 3.34   &  3.21   \\
\ourname{}  & 3.60   &  3.28   \\
SGD         & 5.00   &  4.78   \\
\bottomrule
\end{tabular}
}
\vspace{5pt}
\caption{Validation loss comparison in pre-training GPT-2 style models on FineWeb \citep{penedo2024the} (2.6B tokens for NanoGPT and 7.8B for GPT-medium). We observe a similar pattern with the main experiments, where the \ourname{} has significantly closer loss to Adam than SGD. Note that SGD (without the added clipping) exhibited unstable training dynamics.}
\end{table}

\subsection{Multiple Random Seed Runs}

\label{apx:seeds}

We extend our results of 130M model in Table \ref{tab:main_results} with different random seed by the following table that records the last iterate validation loss. Every run uses the same hyperparameters as listed in \ref{apx:experimental_settings}.

\begin{table}[h]
\centering
\begin{tabular}{lcccc}
\toprule
 Seed & 42 & 1042 & 2042 & 3042 \\
\midrule
Adam        & 3.11 & 3.11 & 3.11 & 3.11 \\
SGD         & 5.07 & 6.13 & 5.05 & 5.05 \\
\ourname{}  & 3.24 & 3.26 & 3.25 & 3.24 \\
\bottomrule
\end{tabular}
\vspace{5pt}
\caption{Validation loss across random seeds. We observe minimal variance for Adam and \ourname{}, while for SGD there is one run that yield even higher loss. This validates that our main results about reducing the gap are consistent and not affected by experimental noise.}
\label{tab:seed_results}
\end{table}

\subsection{Overtraining experiments}
\label{apx:overtraining}

We performed additional experiments with extended number of tokens reaching 20 times more than the Chinchilla compute optimal ratio \citep{hoffmann2022training}  for the 130M model using the same settings as in our main experiments and we show the results bellow:

\vspace{25pt}

{\textbf{4$\times$ Chinchilla}
\vspace{-5pt}
\begin{table}[H]
\vspace{5pt}
\label{tab:overtraining_4x}
\centering
\small
\resizebox{0.55\textwidth}{!}{%
\begin{tabular}{ccc}
\toprule
                      & Val. Loss      & Relative Loss Increase from Adam    \\
\midrule
Adam               &  2.94    & $0\%$   \\
SGD                &  4.04     & $37.41\%$  \\
\ourname{}              &  3.03    & $3.06\%$   \\
\bottomrule
\end{tabular}
}
\vspace{5pt}
\caption{Final validation losses and relative loss increase from Adam when training at 4 times the Chinchilla tokens, i.e., 10.4B tokens.
}
\end{table}
}

\textbf{8$\times$ Chinchilla}
\vspace{-5pt}
\begin{table}[H]
\vspace{5pt}
\label{tab:overtraining_8x}
\centering
\small
\resizebox{0.55\textwidth}{!}{%
\begin{tabular}{ccc}
\toprule
                      & Val. Loss      & Relative Loss Increase from Adam    \\
\midrule
Adam               &  2.89    & $0\%$   \\
SGD                &  3.70     & $28.03\%$  \\
\ourname{}              &  2.95    & $2.08\%$   \\
\bottomrule
\end{tabular}
}
\vspace{5pt}
\caption{Final validation losses and relative loss increase from Adam when training at 8 times the Chinchilla tokens, i.e., 20.8B tokens.
}
\end{table}

\textbf{20$\times$ Chinchilla}
\vspace{-5pt}
\begin{table}[H]
\label{tab:overtraining_20x}
\centering
\small
\resizebox{0.55\textwidth}{!}{%
\begin{tabular}{ccc}
\toprule
                      & Val. Loss      & Relative Loss Increase from Adam    \\
\midrule
Adam               &  2.85    & $0\%$   \\
SGD                &  3.40    & $19.30\%$  \\
\ourname{}             &  2.89    & $1.40\%$   \\
\bottomrule
\end{tabular}
}
\vspace{5pt}
\caption{Final validation losses and relative loss increase from Adam when training at 20 times the Chinchilla tokens, i.e., 52B tokens.
}
\end{table}

In this highly overtrained regime we observe a similar pattern as in our main experiments, that, after addressing the gradient irregularities, running SGD with large LR (i.e., LR=100) is able to close most of the gap to Adam. Note that SGD (without the gradient clipping) diverges for LR greater than 0.5, similar to our main experiments. We did not observe any problems (such as loss spikes or output activation issues) with \ourname{}  at this scale.

\subsection{Downstream results}

\label{apx:downstream}

\begin{table}[H]
    \centering
    \begin{tabular}{lrrr}
        \toprule
        Task           & Adam   & SGD-LL & SGD    \\
        \midrule
        arc\_challenge & 0.3055 & 0.2867 & 0.2022 \\
        arc\_easy      & 0.6473 & 0.6242 & 0.2942 \\
        boolq          & 0.6003 & 0.5679 & 0.3783 \\
        hellaswag      & 0.5181 & 0.4881 & 0.2728 \\
        openbookqa     & 0.2660 & 0.2600 & 0.1780 \\
        piqa           & 0.7606 & 0.7628 & 0.6192 \\
        sciq           & 0.8860 & 0.8510 & 0.1160 \\
        winogrande     & 0.6259 & 0.5762 & 0.4925 \\
        \midrule
        Average        & 0.5762 & 0.5521 & 0.3192 \\
        Relative drop from Adam        & $0\%$ & $4.18\%$ & $44.60\%$ \\
        \bottomrule
    \end{tabular}
    \vspace{5pt}
    \caption{Zero-shot downstream accuracy on 8 tasks of LM-Eval \citep{eval-harness} using the 7B models we pre-trained on 100B C4 tokens using Adam, SGD-LL and SGD. The downstream results follow a similar pattern to the pre-training ones, as enabling large learning rate for SGD reduces  the downstream gap to Adam (relative drop in accuracy) from $44.60\%$ to just $4.18\%$. }
    \label{tab:downstream_accuracy}
\end{table}

\subsection{Distance from initialization experiments}

\label{apx:distance_from_init}

 In the Table below we record the distance from initialization $\|\mathbf{w}^{T} - \mathbf{w}^{0}\|_F$ by each of the three approaches for the 130M-1B models. We can see that the large learning rate enables SGD to move substantially farther from initialization, closer than Adam's parameter displacement.

\begin{table}[H]
\centering
\renewcommand{\arraystretch}{1.3}
\setlength{\tabcolsep}{12pt}
\begin{tabular}{l|c|c|c}
\toprule
&  130M (2.6B tokens)
&  350M (7.8B tokens)
&  1B$^*$ (20B tokens) \\
\midrule
Adam
& $1.3 \times 10^{3}$  
& $2.7 \times 10^{3} $ 
& $7.2 \times 10^{4}$ \\ 
\midrule
SGD
& $3.4 \times 10^{1}$
& $6.2 \times 10^{1}$ 
& $6.2 \times 10^{2}$ \\ 
\midrule
\ourname{}
& $3.9 \times 10^{2}$ 
& $9.5 \times 10^{2}$
& $1.3 \times 10^{4}$ \\ 
\bottomrule
\end{tabular}
\vspace{5pt}
\caption{Distance from initialization, measured by
$\|\mathbf{w}^{T} - \mathbf{w}^{0}\|_F$. $\mathbf{w}^{0}$ represent the weights at initialization and $\mathbf{w}^{T}$ the solution reached by the end of training. $^*$For 1B, distances are summed over parameter tensors during training rather than combined into a full-vector norm.
}
\label{tab:distance_from_init}
\end{table}

\section{Further Connections Between the Large Weight-SG Ratio and Theoretical Analysis}\label{apx:theory_justification}
In this section, we provide more evidence in relation between Theorem \ref{thm:full_grad_mainbody_unified} and the large weight-SG ratio in LLM training. 

We start by checking that the terms $\varepsilon_j, p_j(\xi), c_q$ in Theorem \ref{thm:full_grad_mainbody_unified} are indeed small if we choose an appropriate partition $\mathcal{H}$ and $\mathcal{T}$. In particular, we check the per-token-class loss and the token frequencies during training. In particular, we separate the tokens into ten quantiles, with quantile 0 containing the highest-frequency tokens and quantile 9 containing the lowest-frequency tokens. For each iteration, we count the token occurrence in the mini-batch and demonstrate the average per-token loss from each quantile. Figure~\ref{fig:quantile_loss} illustrates the per-token loss and Figure~\ref{fig:quantile_count} illustrates the token counts across different frequency quantiles. In this example, we can cluster tokens from quantile 0--7 into the head set $\mathcal{H}$, exhibiting higher frequencies and lower loss, while clustering tokens from quantile 8--9 into the tail set $\mathcal{T}$, exhibiting lower frequencies and higher loss. Therefore, we can apply the bias-variance decomposition \eqref{eq:tokenwise_grad_batch_sec_mainbody} to achieve a smaller upper bound as the head tokens are better-trained. Besides, we can apply the background-hit decomposition \eqref{eq:tokenwise_grad_rare_mainbody} to the tail tokens to achieve a smaller upper bound since they are rare (i.e., small $q_j$) and under-trained (i.e., small $p_j(\xi)$).

\begin{figure}[H]
    \centering
\includegraphics[width=0.6\linewidth]{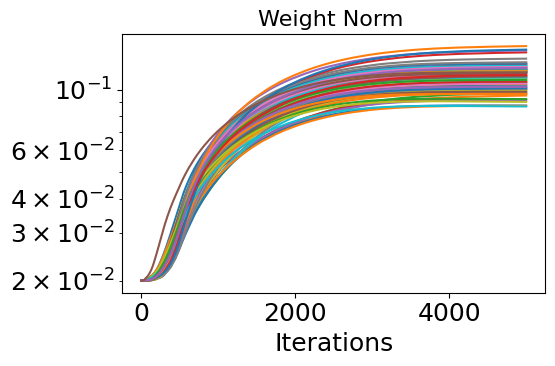}
\vspace{-.25cm}
    \caption{The  RMS norms of the weight matrices during training.}
    \label{fig:weight_norm}
\vspace{-.5cm}
\end{figure}

\begin{figure}
    \centering
    \includegraphics[width=0.85\linewidth]{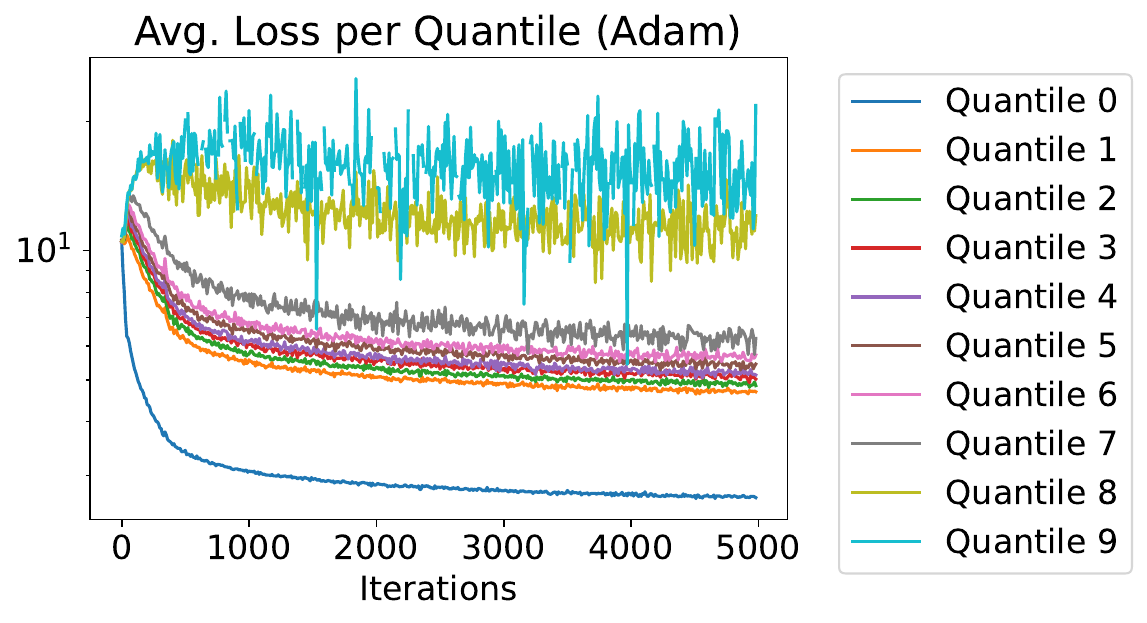}\\
\vspace{-.25cm}
    \caption{Average loss of the tokens with different frequency quantiles, with Quantile 0 representing the top-10\% most frequent token classes and Quantile 9 representing the bottom-10\% least frequent token classes. High-frequency tokens exhibits significantly lower loss.} 
    \label{fig:quantile_loss}
        \vspace{-0.5cm}
\end{figure}

By the discussion above, we conclude that $\|G\|_F^2$ is small in LLM training. On the other hand, the norm of the weight matrix is typically lower bounded (see Figure~\ref{fig:weight_norm}). Therefore,
the weight-SG ratio becomes large. Moreover, we note that the upper bound \eqref{eq:full_grad_upper_bound_universal_unified} is a decreasing function w.r.t. $B$. Therefore, we expect that the weight-SG ratio should increase as we enlarge the batch size, which is indeed consistent with the observation in Figure~\ref{fig:w_sg_ratio_training}(c) and Table~\ref{tab:weight_sg_ratio_stat}.

\begin{figure}
    \centering
    \includegraphics[width=0.7\linewidth]{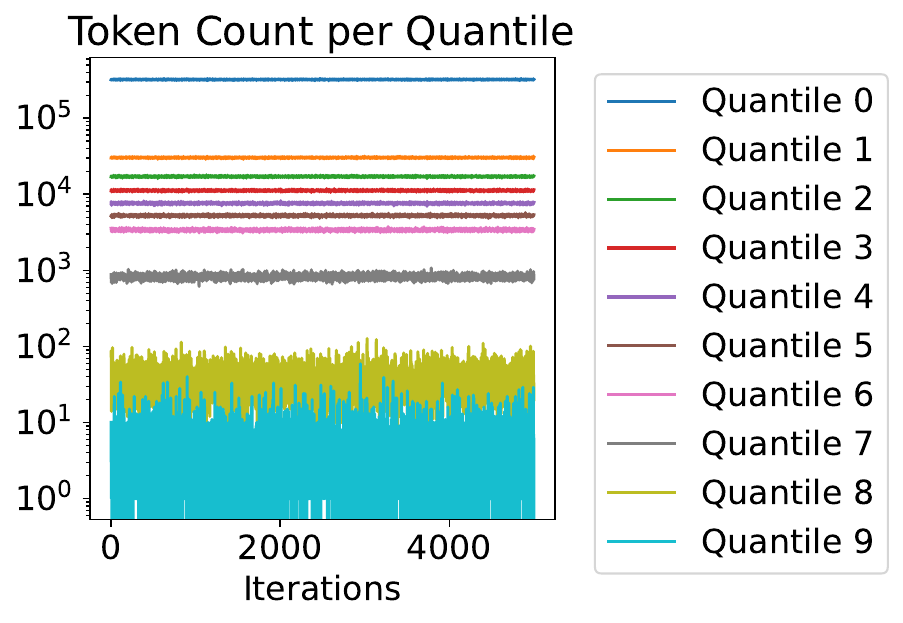}
\vspace{-.25cm}
    \caption{Token count in one batch on different frequency quantiles, with Quantile 0 representing the top-10\% most frequent token classes and Quantile 9 representing the bottom-10\% least frequent token classes. Low-frequency tokens exhibits substantially rare occurrence.}
    \label{fig:quantile_count}
    \vspace{-0.5cm}
\end{figure}

The large weight-SG norm ratio implies that a small choice of learning rate would limit the optimization algorithm's solution space of ${\bf w}^t$ to a small neighborhood around the initialization state ${\bf w}^0$. It is evident that the usual learning rate choice of SGD in the large batch setting will leave a significant performance gap to adaptive algorithms \citep{pan2023toward, marek2025small,sreckovic2025your}. These observations suggest that the limited learning rates constrain SGD’s training progress in the large-batch setting, motivating our investigation of targeted clipping as a way of enabling larger updates.

\begin{table}
\caption{Statistics of weight-SG ratio over the training trajectory of different batch sizes in Figure \ref{fig:w_sg_ratio_training}(c). As batch size increases, the weight-SG ratio increases.}
\label{tab:weight_sg_ratio_stat}
\centering
\small
\begin{tabular}{l|ccc}
\toprule
{\bf Batch Size} & {\bf Min}  & {\bf Max} & {\bf Average} \\
\midrule
$B=64$  & $ 1.5 \times 10^3$ & $ 1.9 \times 10^4$ & $ 5.9 \times 10^3$ \\
$B=2048$  & $ 2.3 \times 10^3$ & $ 2.2 \times 10^5$ & $ 2.3 \times 10^4$  \\
$B=65536$  & $ 2.6 \times 10^3$ & $ 4.9 \times 10^5$ & $ 5.8 \times 10^4$  \\
\bottomrule

\end{tabular} \vspace{-0.5cm}
\end{table}

\section{Experimental Analysis of Magnitude of $\kappa$ in Theorem \ref{thm:full_grad_mainbody_unified}}\label{apx:kappa_experiments}
 
The constant $\kappa$ in Theorem \ref{thm:full_grad_mainbody_unified} can be up to $|V|$ in the worst case. However, we observe that $\kappa$ is small in practical pre-training settings. Below we present an experimental estimation of $\kappa$ for different layers.

In particular, we estimate $\kappa$ by calculating the ratio of $\|G\|_F^2/\sum_j\|G_j\|_F^2$, where $G_j$ is the contribution of token class $j$ to the stochastic gradient $G=\sum_j G_j$. Directly calculating each $G_j$ would result in $|V|=32,000$ back-propagations for one estimation (where $|V|$ is the vocabulary size). To make the computation tractable, we apply a random-estimation approach to estimate $\sum_j\|G_j\|_F^2$. Specifically, let $r_j$ be independent Rademacher variables, i.e., $r_j\in\{-1, +1\}$. Define $S_r=\sum_{j=1}^V r_jG_j$, then $\mathbb{E}_r\|S_r\|_F^2=\sum_j\|G_j\|_F^2$. We therefore sample $S_r$ by $K$ times to obtain an unbiased estimation of the sum $\Sigma_j\|G_j\|_F^2$, which only requires $K$ times of back-propagation (we use $K=16$ in our experiments).

We tested on the LLaMA-130M model (with 12 total transformer layers, same setting as our main experiments) with both Adam and SGD. We tracked the gradient of the following four layers: the input embedding layer, Layer 1, Layer 6 and Layer 11. The observed $\|G\|_F^2/\Sigma_j\|G_j\|_F^2$ ratios are presented in the following table:

\noindent\textbf{Adam}

\begin{center}
\begin{tabular}{lrrrrrr}
\toprule
\textbf{Iteration} & \textbf{0} & \textbf{1000} & \textbf{2000} & \textbf{3000} & \textbf{4000} & \textbf{5000} \\
\midrule
Embedding & 0.98 & 1.85 & 1.37 & 1.30 & 1.03 & 1.01 \\
Layer 1   & 1.00 & 2.75 & 1.81 & 1.58 & 1.16 & 1.08 \\
Layer 6   & 0.98 & 2.48 & 2.05 & 1.75 & 1.47 & 1.22 \\
Layer 11  & 1.00 & 1.82 & 1.44 & 1.18 & 0.97 & 0.88 \\
\bottomrule
\end{tabular}
\end{center}

\noindent\textbf{SGD}

\begin{center}
\begin{tabular}{lrrrrrr}
\toprule
\textbf{Iteration} & \textbf{0} & \textbf{1000} & \textbf{2000} & \textbf{3000} & \textbf{4000} & \textbf{5000} \\
\midrule
Embedding & 0.98 & 1.03 & 0.96 & 0.93 & 0.89 & 0.86 \\
Layer 1   & 1.00 & 1.90 & 1.33 & 1.24 & 1.02 & 0.94 \\
Layer 6   & 0.98 & 1.15 & 1.01 & 0.95 & 0.84 & 0.80 \\
Layer 11  & 1.00 & 0.82 & 0.80 & 0.79 & 0.72 & 0.72 \\
\bottomrule
\end{tabular}
\end{center}

Across all measured layers, optimizers, and checkpoints, the estimated $\|G\|_F^2/\Sigma_j\|G_j\|_F^2$ ratio lies between 0.72 and 2.75, being considerably smaller than the worst case $|V|=32$K. These results provide empirical evidence that the worst-case $\kappa=V$ is indeed highly conservative in practical settings for the embedding and for the intermediate transformer layers. Also note that for the output layer, $\kappa$ is exactly 1. Therefore, for all layers we experimentally estimate $\kappa$ to be small, which provides a more concrete connection between our theoretical result and the experimental observations.

\section{Characterization of the Output-Layer Gradient Norm}\label{apx:proof}

In this section, we characterize the gradient norm of the output-layer weight matrix from a token-wise perspective. We begin by analyzing the per-token-class gradient. In particular, by decomposing the per-token-class gradient in different ways, we obtain different bounds that apply to different regimes. We then study how gradient behavior depends on token heterogeneity. Specifically, we show that the per-token-class gradient norm is (i) positively correlated with token frequency in the early stage of training, and (ii) negatively correlated with the batch size in the late stage of training. Details are provided below.

Let $V$ denote the vocabulary size and $\Delta^{V-1}$ denote the unit simplex in $\mathbb{R}^V$. The weight matrix of the output-layer is $W \in \mathbb{R}^{V \times d}$, and its $j$-th row is denoted by $w_j^\top$.
Given a hidden representation $h \in \mathbb{R}^d$, the logits and the predictive distribution are

\[
z = Wh, \qquad p(h) = \mathrm{softmax}(Wh) \in \Delta^{V-1},
\]
where $p_j(h)$ denotes the $j$-th coordinate.
For a true label $y \in \{1,\dots,V\}$, the cross-entropy loss is
\[
\ell(h,y) = -\log p_y(h).
\]

Therefore, the gradient with respect to the $j$-th row $w_j$ for a single sample $(h,y)$ is
\begin{equation}\label{eq:tokenwise_grad}
    g_j(h,y) \coloneqq \nabla_{w_j} \ell(h,y)
= \bigl(p_j(h) - \mathbf{1}\{y=j\}\bigr)\,h \in \mathbb{R}^d .
\end{equation}

Define $r_j(h,y)\coloneqq p_j(h)-\mathbf{1}\{y=j\}$. Then we have $g_j(h,y)=r_j(h,y)h$.

Let $(h,y)\sim \mathcal{D}$ be a generic data-generating distribution and let $\{(h_b,y_b)\}_{b=1}^B$ be a batch sampled i.i.d.\ from $\mathcal{D}$. The batch-averaged gradient over $w_j$ is thus
\[
\bar g_j \coloneqq \frac{1}{B}\sum_{b=1}^B g_j(h_b,y_b).
\]

We allow the conditional label distribution to depend on $h$:
\begin{equation}\label{eq:tokenwise_true_prob}
    q_j(h) \coloneqq \mathbb{P}(y=j\mid h), \qquad q(h)\in\Delta^{V-1}.
\end{equation}

The factorized structure of the per-token-class gradient allows us to separate it into a bias term driven by the model–data mismatch $p_j(h)-q_j(h)$ and a variance term induced by the label randomness.

\begin{lemma}\label{lem:tokenwise_grad_exp}
For any fixed $h$, we have
\[
\mathbb{E}[r_j(h,y)\mid h] = p_j(h)-q_j(h), \qquad 
\mathbb{E}[g_j(h,y)\mid h] = \bigl(p_j(h)-q_j(h)\bigr)\,h .
\]
\end{lemma}

\begin{proof}
By \eqref{eq:tokenwise_grad} and \eqref{eq:tokenwise_true_prob}, we have $g_j(h,y)=(p_j(h)-\mathbf{1}\{y=j\})h$ and
$\mathbb{E}[\mathbf{1}\{y=j\}\mid h]=q_j(h)$. Substituting these identities yields the claim.
\end{proof}

\begin{lemma}
For any fixed $h$,
\begin{equation}\label{eq:tokenwise_residual}
    \mathbb{E}[r_j(h,y)^2\mid h]
= q_j(h) - 2q_j(h)p_j(h) + p_j(h)^2 .
\end{equation}
Consequently,
\begin{equation}\label{eq:tokenwise_grad_second_moment}
\mathbb{E}\bigl[\|g_j(h,y)\|^2 \mid h\bigr]
= \bigl(q_j(h) - 2q_j(h)p_j(h) + p_j(h)^2\bigr)\,\|h\|^2 .
\end{equation}
\end{lemma}

\begin{proof}
Conditioned on $h$, the event $\{y=j\}$ occurs with probability $q_j(h)$,
in which case $r_j(h,y)=p_j(h)-1$; otherwise $r_j(h,y)=p_j(h)$.
Computing the conditional expectation gives 
\begin{equation*}
    \mathbb{E}[r_j(h,y)^2\mid h]
= q_j(h)\bigl(1-p_j(h)\bigr)^2 + \bigl(1-q_j(h)\bigr)p_j(h)^2
= q_j(h) - 2q_j(h)p_j(h) + p_j(h)^2 .
\end{equation*}
Since $g_j(h, y)=r_j(h,y)h$, \eqref{eq:tokenwise_grad_second_moment} follows immediately.
\end{proof}

\begin{lemma}\label{lem:tokenwise_grad_cov}
The conditional covariance of $g_j(h,y)$ given $h$ is
\begin{equation}\label{eq:tokenwise_grad_cov}
    \mathrm{Cov}\bigl(g_j(h,y)\mid h\bigr)
= q_j(h)\bigl(1-q_j(h)\bigr)\, h h^\top .
\end{equation}
\end{lemma}

\begin{proof}
We first notice that 
\begin{equation*}
    \mathrm{Var}(r_j(h,y) \mid h)=\mathbb{E}[r_j(h,y)^2\mid h]-\mathbb{E}[r_j(h,y)\mid h]^2=q_j(h)(1-q_j(h)).
\end{equation*}
Since $g_j(h,y)=r_j(h,y)h$ and $h$ is fixed under conditioning,
\[
\mathrm{Cov}\bigl(g_j(h,y)\mid h\bigr)
= \mathrm{Var}(r_j(h,y)\mid h)\, h h^\top
= q_j(h)\bigl(1-q_j(h)\bigr)\, h h^\top .
\]
\end{proof}

Lemmas~\ref{lem:tokenwise_grad_exp}--\ref{lem:tokenwise_grad_cov} characterize the single-sample per-token-class gradient. We next derive bounds for the batch-averaged per-token-class gradient $\bar g_j$ under different decompositions and regimes.

\textbf{Case 1: Well-Trained Regime.} 

Assume that token $j$ is well trained so that $p_j(h)$ is sufficiently close to $q_j(h)$. In this case, $\mathbb{E}[g_j]$ is small, and the batch-averaged gradient norm is dominated by the covariance term. We begin with the standard bias--variance decomposition.

\begin{lemma}\label{lem:tokenwise_grad_decomposition}
Assume that $\{(h_b,y_b)\}_{b=1}^B$ are sampled i.i.d.\ from $\mathcal{D}$, and that
$\mathbb{E}\|h\|^2<\infty$. Let $\mu_j\coloneqq \mathbb{E}[g_j]$. Then
\begin{equation}\label{eq:tokenwise_grad_batch_sec}
    \mathbb{E}\bigl[\|\bar g_j\|^2\bigr]
= \|\mu_j\|^2 + \mathrm{tr}\bigl(\mathrm{Cov}(\bar g_j)\bigr),
\end{equation}
where
\begin{equation}\label{eq:tokenwise_grad_batch_cov}
\mathrm{Cov}(\bar g_j)
= \frac{1}{B}\Bigl(
\mathbb{E}\bigl[q_j(h)(1-q_j(h))\,hh^\top\bigr]
+ \mathrm{Cov}\bigl((p_j(h)-q_j(h))h\bigr)
\Bigr).
\end{equation}
\end{lemma}

\begin{proof}
Since $\{(h_b,y_b)\}_{b=1}^B$ are i.i.d.\ sampled, we have
\[
\mathbb{E}[\bar g_j] = \mathbb{E}[g_j], \qquad
\mathrm{Cov}(\bar g_j)=\frac{1}{B}\,\mathrm{Cov}(g_j).
\]
Then \eqref{eq:tokenwise_grad_batch_sec} follows the standard bias--variance decomposition, and \eqref{eq:tokenwise_grad_batch_cov} follows from the law of total variance applied to $g_j$ conditional on $h$, together with Lemma \ref{lem:tokenwise_grad_exp} and \ref{lem:tokenwise_grad_cov}.
\end{proof}

Given the representation of the per-token-class gradient in Lemma \ref{lem:tokenwise_grad_decomposition}, we are now able to bound its norm.

\begin{proposition}[Well Trained Regime Scaling]\label{prop:grad_head}
Assume $\mathbb{E}\|h\|^2<\infty$. Suppose that there exists $\varepsilon\ge 0$ such that
\[
\mathbb{E}\bigl[\|(p_j(h)-q_j(h))h\|^2\bigr]\le \varepsilon^2 .
\]
Then
\[
\mathbb{E}\|\bar g_j\|^2
\le \Bigl(1+\frac{1}{B}\Bigr)\varepsilon^2
+\frac{1}{B}\,\mathbb{E}\bigl[q_j(h)\bigl(1-q_j(h)\bigr)\|h\|^2\bigr].
\]

\end{proposition}

\begin{proof}
From Lemma \ref{lem:tokenwise_grad_decomposition} we have
\[
\operatorname{Cov}(\bar g_j)
=\frac{1}{B}\Bigl(
\mathbb{E}\bigl[q_j(h)(1-q_j(h))\,hh^\top\bigr]
+\operatorname{Cov}\bigl((p_j(h)-q_j(h))h\bigr)
\Bigr).
\]
Taking traces on $\operatorname{Cov}(\bar g_j)$ gives
\begin{equation}\label{eq:cor1_trace_cov}
\operatorname{tr}\bigl(\operatorname{Cov}(\bar g_j)\bigr)
=\frac{1}{B}\Bigl(
\mathbb{E}\bigl[q_j(h)(1-q_j(h))\,\|h\|^2\bigr]
+\operatorname{tr}\bigl(\operatorname{Cov}\bigl((p_j(h)-q_j(h))h\bigr)\bigr)
\Bigr).
\end{equation}
Recall that $\mu_j=\mathbb{E}[g_j]=\mathbb{E}[(p_j(h)-q_j(h))h]$, hence by Jensen's inequality,
\begin{equation}\label{ineq:cor1_grad_exp}
\|\mu_j\|^2=\|\mathbb{E}[(p_j(h)-q_j(h))h]\|^2\le \mathbb{E}\|(p_j(h)-q_j(h))h\|^2\le \varepsilon^2.
\end{equation}
Furthermore, we have
\begin{equation}\label{ineq:cor1_trace_cov}
    \operatorname{tr}(\operatorname{Cov}\bigl((p_j(h)-q_j(h))h\bigr))
=\mathbb{E}\|(p_j(h)-q_j(h))h\|^2-\|\mathbb{E}[(p_j(h)-q_j(h))h]\|^2
\le \mathbb{E}\|(p_j(h)-q_j(h))h\|^2\le \varepsilon^2.
\end{equation}

Combining \eqref{eq:tokenwise_grad_batch_sec}, \eqref{eq:cor1_trace_cov}, \eqref{ineq:cor1_grad_exp} and \eqref{ineq:cor1_trace_cov} yields
\[
\mathbb{E}\|\bar g_j\|^2
\le \varepsilon^2
+\frac{1}{B}\Bigl(
\mathbb{E}[q_j(h)(1-q_j(h))\|h\|^2]+\varepsilon^2
\Bigr),
\]
which proves the claim.
\end{proof}

\textbf{Remark:} In particular, if $\varepsilon$ is negligible, then $\mathbb{E}\|\bar g_j\|^2 = O(1/B)$. The per-token-class gradient norm $\mathbb{E}\|\bar g_j\|$ is therefore in a $O(1/\sqrt{B})$ scaling in this case.

\textbf{Case 2: Rare Tokens.}

Assume that token $j$ has low marginal frequency. In this case, the gradient norm can be characterized via a different decomposition. Specifically, the batch-averaged gradient admits

\begin{equation}\label{eq:tokenwise_grad_rare}
\bar g_j=\frac{1}{B}\sum_{b=1}^B (p_j(h_b)-\mathbf{1}\{y_b=j\})h_b
=\underbrace{\frac{1}{B}\sum_{b=1}^B p_j(h_b)\,h_b}_{\text{``softmax background''}}-
\underbrace{\frac{1}{B}\sum_{b:y_b=j} h_b}_{\text{``label hits''}}.
\end{equation}

Since $p_j(h_b)$ is typically small for rare tokens in the under-trained regime, $\bar g_j$ is often dominated by the ``label hits'' term. In particular, we obtain the following bound.

\begin{proposition}[Rare Tokens]\label{prop:grad_tail}
Let $q_j \coloneqq \mathbb{P}(y=j)$. 
Then the batch-averaged per-token-class gradient satisfies
\begin{equation}
\label{eq:weak_bg_sharp_second_moment_bound_inline}
\mathbb{E}\|\bar g_j\|^2
\;\le\;
\frac{2q_j}{B}\,\mathbb{E}\bigl[\|h\|^2 \mid y=j\bigr]
\;+\;
2q_j^2\bigl\|\mathbb{E}[h \mid y=j]\bigr\|^2
\;+\;
2\mathbb{E}\bigl[p_j(h)^2\|h\|^2\bigr].
\end{equation}
\end{proposition}

\begin{proof}
We bound $\mathbb{E}\|\bar g_j\|^2$ using the decomposition in \eqref{eq:tokenwise_grad_rare}.

First, for the softmax background term, since $\{h_b\}_{b=1}^B$ are i.i.d. samples, we have 
\begin{equation}\label{ineq:cor2_mean}
\mathbb{E}\left\|\frac{1}{B}\sum_{b=1}^B p_j(h_b)\,h_b\right\|^2
\le\;\frac{1}{B}\sum_{b=1}^B\mathbb{E}\left\| p_j(h_b)\,h_b\right\|^2
=
\mathbb{E}\|p_j(h)h\|^2
=
\mathbb{E}\bigl[p_j(h)^2\|h\|^2\bigr].
\end{equation}

Second, for the label-hit term, 
let $v \coloneqq \mathbf{1}\{y=j\}\,h$. Then
\[\frac{1}{B}\sum_{b:y_b=j} h_b = \frac{1}{B}\sum_{b=1}^B v_b\] 
with i.i.d.\ copies $v_b$ of $v$. Therefore,
\[
\mathbb{E}\left\|\frac{1}{B}\sum_{b:y_b=j} h_b\right\|^2
=
\frac{1}{B^2}\Bigl(
B\,\mathbb{E}\|v\|^2 + B(B-1)\|\mathbb{E}v\|^2
\Bigr)
=
\frac{1}{B}\,\mathbb{E}\|v\|^2
+
\frac{B-1}{B}\,\|\mathbb{E}v\|^2.
\]
By the law of total expectation,
\[
\mathbb{E}\|v\|^2
=
\mathbb{E}\bigl[\mathbf{1}\{y=j\}\|h\|^2\bigr]
=
q_j\,\mathbb{E}\bigl[\|h\|^2 \mid y=j\bigr],
\]
and
\[
\mathbb{E}v
=
\mathbb{E}[\mathbf{1}\{y=j\}h]
=
q_j\,\mathbb{E}[h \mid y=j].
\]
Hence,
\begin{equation}\label{ineq:cor2_label}
\mathbb{E}\left\|\frac{1}{B}\sum_{b:y_b=j} h_b\right\|^2
=
\frac{q_j}{B}\,\mathbb{E}\bigl[\|h\|^2 \mid y=j\bigr]
+
\frac{B-1}{B}\,
q_j^2\bigl\|\mathbb{E}[h \mid y=j]\bigr\|^2
\;\le\;
\frac{q_j}{B}\,\mathbb{E}\bigl[\|h\|^2 \mid y=j\bigr]
+
q_j^2\bigl\|\mathbb{E}[h \mid y=j]\bigr\|^2.
\end{equation}

Combining \eqref{ineq:cor2_mean} and \eqref{ineq:cor2_label} yields
\[
\mathbb{E}\|\bar g_j\|^2
\le
2\mathbb{E}\left\|\frac{1}{B}\sum_{b=1}^B p_j(h_b)\,h_b\right\|^2 + 2\mathbb{E}\left\|\frac{1}{B}\sum_{b:y_b=j} h_b\right\|^2
\le
2\mathbb{E}\bigl[p_j(h)^2\|h\|^2\bigr]
+
\frac{2q_j}{B}\,\mathbb{E}\bigl[\|h\|^2 \mid y=j\bigr]
+
2q_j^2\bigl\|\mathbb{E}[h \mid y=j]\bigr\|^2,
\]
which proves \eqref{eq:weak_bg_sharp_second_moment_bound_inline}.
\end{proof}

\subsection{Analysis of Full Gradient}
Given the token-wise bounds above, we now analyze the full gradient matrix.

Let
\[
G \;\coloneqq\; \frac{1}{B}\sum_{b=1}^B \bigl(p(h_b)-e_{y_b}\bigr)h_b^\top
\;\in\; \mathbb{R}^{V\times d}
\]
denote the mini-batch stochastic gradient of the output-layer weight matrix, where the
$j$-th row of $G$ is $\bar g_j^\top$.

Then
\[
\|G\|_F^2 = \sum_{j=1}^V \|\bar g_j\|^2,
\qquad
\mathbb{E}\|G\|_F^2 = \sum_{j=1}^V \mathbb{E}\|\bar g_j\|^2 .
\]
where $\bar g_1^\top,\dots,\bar g_V^\top$ are the row (per-token-class) vectors of $G$.
Let $q_j \coloneqq \mathbb{P}(y=j)$ denote the marginal frequency of token $j$.
For each token $j$, define
\[
m_j \coloneqq \mathbb{E}[h\mid y=j],
\qquad
s_j^2 \coloneqq \mathbb{E}[\|h\|^2\mid y=j],
\qquad
\Sigma_j \coloneqq \mathrm{Cov}(h\mid y=j),
\]
so that $s_j^2=\|m_j\|^2+\operatorname{tr}(\Sigma_j)$.

We partition the vocabulary into a head set $\mathcal H$ and a tail set
$\mathcal T$.
We then upper bound the full gradient norm by applying the two token-wise bounds to the two sets, respectively.

\begin{theorem}\label{thm:full_grad}
Let
\[
G \;\coloneqq\; \frac{1}{B}\sum_{b=1}^B \bigl(p(h_b)-e_{y_b}\bigr)h_b^\top
\;\in\; \mathbb{R}^{V\times d}
\]
be the mini-batch stochastic gradient of the output-layer weight matrix. Suppose that the tokens $\{1, ..., V\}$ are separated into two sets $\mathcal{H}$ and $\mathcal{T}$. Denote $c_q\;\coloneqq\;\underset{j\in\mathcal{T}}{\max}~q_j$. Further, denote
\[
\varepsilon_j^2\;\coloneqq\;\mathbb{E}\bigl[\|(p_j(h)-q_j(h))h\|^2\bigr].
\]

Then
\begin{equation}\label{eq:full_grad_upper_bound_universal}
\mathbb{E}\|G\|_F^2
\;\le\;
\Bigl(1+\frac{1}{B}\Bigr)\sum_{j\in\mathcal H}\varepsilon_j^2
+\frac{1}{B}\,\mathbb{E}\|h\|^2
+
\left(\frac{2}{B}+2c_q\right)\sum_{j\in\mathcal T} q_j s_j^2
+
2\sum_{j\in\mathcal{T}}\mathbb{E}\bigl[p_j(h)^2\|h\|^2\bigr].
\end{equation}
\end{theorem}

\begin{proof}
Since $\|G\|_F^2=\sum_{j=1}^V \|\bar g_j\|^2$, it suffices to bound
$\sum_{j\in\mathcal H}\mathbb{E}\|\bar g_j\|^2$ and
$\sum_{j\in\mathcal T}\mathbb{E}\|\bar g_j\|^2$ separately.

For $j\in\mathcal H$, we apply Proposition~\ref{prop:grad_head} to obtain
\[
\mathbb{E}\|\bar g_j\|^2
\le
\Bigl(1+\frac{1}{B}\Bigr)\varepsilon_j^2
+\frac{1}{B}\,\mathbb{E}\bigl[q_j(h)\bigl(1-q_j(h)\bigr)\|h\|^2\bigr].
\]
Summing over $j\in\mathcal H$ gives
\begin{equation}\label{eq:full_grad_head}
\sum_{j\in\mathcal H}\mathbb{E}\|\bar g_j\|^2
\le
\Bigl(1+\frac{1}{B}\Bigr)\sum_{j\in\mathcal H}\varepsilon_j^2
+\frac{1}{B}\sum_{j\in\mathcal H}\mathbb{E}\bigl[q_j(h)\bigl(1-q_j(h)\bigr)\|h\|^2\bigr].
\end{equation}
We now bound the second term in \eqref{eq:full_grad_head}, since $\mathcal H\subseteq\{1,\dots,V\}$ and
$q_j(h)(1-q_j(h))\le q_j(h)$, we have
\[
\sum_{j\in\mathcal H} q_j(h)\bigl(1-q_j(h)\bigr)
\le
\sum_{j=1}^V q_j(h)\bigl(1-q_j(h)\bigr)
=
1-\sum_{j=1}^V q_j(h)^2
\le 1.
\]
Therefore,
\[
\sum_{j\in\mathcal H}\mathbb{E}\bigl[q_j(h)\bigl(1-q_j(h)\bigr)\|h\|^2\bigr]
=
\mathbb{E}\Bigl[\sum_{j\in\mathcal H} q_j(h)(1-q_j(h))\|h\|^2\Bigr]
\le
\mathbb{E}\|h\|^2,
\]
and hence the head contribution satisfies
\[
\sum_{j\in\mathcal H}\mathbb{E}\|\bar g_j\|^2
\le
\Bigl(1+\frac{1}{B}\Bigr)\sum_{j\in\mathcal H}\varepsilon_j^2+\frac{1}{B}\mathbb{E}\|h\|^2.
\]

For $j\in\mathcal T$, applying Proposition~\ref{prop:grad_tail} yields
\[
\mathbb{E}\|\bar g_j\|^2
\le
\frac{2q_j}{B}\,s_j^2
+2q_j^2\|m_j\|^2
+2\mathbb{E}\bigl[p_j(h)^2\|h\|^2\bigr].
\]
Summing over $j\in\mathcal T$ gives
\[
\sum_{j\in\mathcal T}\mathbb{E}\|\bar g_j\|^2
\le
\frac{2}{B}\sum_{j\in\mathcal T} q_j s_j^2
+2\sum_{j\in\mathcal T}q_j^2\|m_j\|^2
+2\sum_{j\in\mathcal T}\mathbb{E}\bigl[p_j(h)^2\|h\|^2\bigr].
\]
Note that for all $j\in\mathcal T$, we have
$q_j \le c_q $, and moreover $\|m_j\|^2\le s_j^2$.
Therefore,
\[
\sum_{j\in\mathcal T} q_j^2\|m_j\|^2
\le
\sum_{j\in\mathcal T} \Bigl(q_j\cdot c_q\Bigr)s_j^2
=
c_q\sum_{j\in\mathcal T} q_j s_j^2,
\]
and hence
\[
\sum_{j\in\mathcal T}\mathbb{E}\|\bar g_j\|^2
\le
\left(\frac{2}{B}+2c_q\right)\sum_{j\in\mathcal T} q_j s_j^2
+
2\sum_{j\in\mathcal T}\mathbb{E}\bigl[p_j(h)^2\|h\|^2\bigr].
\]

Combining the head and tail bounds completes the proof.
\end{proof}

Theorem~\ref{thm:full_grad} upper bounds the full gradient norm by separating tokens into head and tail sets. This bound is universal in the sense that it allows arbitrary token-wise heterogeneity and does not depend on the optimization algorithm. The bound shows that the full gradient norm is controlled by (i) the training status of head tokens (captured by $\varepsilon_j$), (ii) the marginal frequencies of tail tokens (captured by $c_q$ and $\{q_j\}_{j\in\mathcal T}$), and (iii) the magnitude of the softmax-background term for tail tokens (captured by $\mathbb{E}[p_j(h)^2\|h\|^2]$). To obtain a simpler scaling law, we introduce the following assumptions.

\begin{assumption}\label{ass:head_tail}
Assume that there exists a constant $0<c_q<1$ such that the tokens $\{1, ..., V\}$ are separated into two sets $\mathcal{H}=\{j\mid q_j\geq c_q\}$ and $\mathcal{T}=\{j\mid q_j< c_q\}$, satisfying the following conditions:
\begin{itemize}
\item (\textbf{Well-trained head}) For $j\in\mathcal H$, there exists
$\varepsilon_j\ge 0$ such that
\[
\mathbb{E}\bigl[\|(p_j(h)-q_j(h))h\|^2\bigr]\le \varepsilon_j^2,
\qquad
\sum_{j\in\mathcal H}\varepsilon_j^2 \le \frac{C_{\mathcal H}}{B}
\]
for some constant $C_{\mathcal H}<\infty$.
\item (\textbf{Under-trained tail}) For $j\in\mathcal T$, there exists a constant $0 < c \ll B$ such that 
\begin{equation*}
\mathbb{E}\bigl[p_j(h)^2\|h\|^2\bigr]
\;\le\;
\frac{c^2}{B^2}\,\mathbb{E}\|h\|^2.
\end{equation*}
\item (\textbf{Tail contains only rare tokens}) There exists a constant
$C_{\mathcal{T}}<\infty$ such that $c_q\leq C_{\mathcal{T}}/B$.
\end{itemize}
\end{assumption}

These assumptions are motivated from the common belief that in LLM pre-training, the head tokens are trained well in a relatively small time, whereas the tail tokens remain under trained for a long time. We in turn justify the assumptions:

\begin{corollary}\label{cor:grad_bound}
    Suppose that Assumption~\ref{ass:head_tail} holds. Then
\begin{equation}\label{eq:full_grad_upper_bound_B}
\mathbb{E}\|G\|_F^2
\;\le\;
\frac{2C_{\mathcal H}}{B}
+\frac{1}{B}\,\mathbb{E}\|h\|^2
+\frac{2(1+C_{\mathcal{T}})}{B}\sum_{j\in\mathcal T} q_j s_j^2
+\frac{2|\mathcal T|c^2}{B^2}\,\mathbb{E}\|h\|^2.
\end{equation}
    In particular, if $\sum_{j\in\mathcal T} q_j s_j^2<\infty$ and $\mathbb{E}\|h\|^2<\infty$,
then $\mathbb{E}\|G\|_F^2 = O(1/B)$.
\end{corollary}
\begin{proof}
    Applying Assumption~\ref{ass:head_tail} to \eqref{eq:full_grad_upper_bound_universal} yields \eqref{eq:full_grad_upper_bound_B}.
\end{proof}

\begin{corollary}\label{cor:wsg_ratio}
    Assume that Assumptions \ref{ass:head_tail} holds and further that the output layer weight matrix $W$ satisfies $\|W\|_F\geq C_w$ for some $C_w>0$. Then 
    \[
    \mathbb{E}\left[\frac{\|W\|_F}{\|G\|_F}\right]=\Omega(\sqrt{B})
    \]
\end{corollary}

\begin{proof}
Denote
\[
C_G
\;\coloneqq\;
\sqrt{2C_{\mathcal H}
+\mathbb{E}\|h\|^2
+2(1+C_{\mathcal{T}})\sum_{j\in\mathcal T} q_j s_j^2+2|\mathcal{T}|c^2\mathbb{E}\|h\|^2} .
\]
By Corollary \ref{cor:grad_bound}, we have $\mathbb{E}\|G\|_F^2\leq C_G^2/B$. Therefore, 
\[
\mathbb{E}\|G\|_F \le \sqrt{\mathbb{E}\|G\|_F^2}
\le \frac{C_G}{\sqrt{B}}.
\]

Since $\|W\|_F\ge C_w$, we have
\[
\mathbb{E}\left[\frac{\|W\|_F}{\|G\|_F}\right]
\ge
C_w\,\mathbb{E}\left[\frac{1}{\|G\|_F}\right]
\ge \frac{C_w}{\mathbb{E}\|G\|_F}
\ge \frac{C_w}{C_G}\sqrt{B},
\]
which completes the proof.
\end{proof}

\subsection{Lower Bound and Imbalance of the Per-Token-Class Gradient}

We next provide a lower bound for the norm of the per-token-class gradient when tokens are not well trained. The result also implies that per-token-class gradient norms can be highly imbalanced across tokens, especially in the early stage of training.

\begin{theorem}\label{thm:grad_lower_and_ratio}
    Assume that there exists a constant $0 < c \ll B$ such that 
\begin{equation}
\label{eq:weak_bg_second_moment_inline_prop}
\mathbb{E}\bigl[p_j(h)^2\|h\|^2\bigr]
\;\le\;
\frac{c^2}{B^2}\,\mathbb{E}\|h\|^2.
\end{equation}
In addition, define $m_j\coloneqq \mathbb{E}[h\mid y=j]$, $s_j^2\coloneqq \mathbb{E}[\|h\|^2\mid y=j]$ and assume
$0<s_j^2<\infty$ for the tokens under consideration.
Then we have
\begin{equation}\label{eq:tokenwise_sec_moment_lower_bound}
\mathbb{E}\|\bar g_j\|^2
\;\ge\;
\frac{1}{2B}\,q_j\,s_j^2
\;-\;
\frac{c^2}{B^2}\,\mathbb{E}\|h\|^2.
\end{equation}

Consequently, for any two tokens $i, j$,
\begin{equation}\label{eq:grad_ratio_bound}
\frac{\mathbb{E}\|\bar g_i\|^2}{\mathbb{E}\|\bar g_j\|^2}
\;\ge\;
\frac{\frac{1}{2B}q_i s_i^2-\frac{c^2}{B^2}\mathbb{E}\|h\|^2}
{\frac{2q_j}{B}\,s_j^2
+2q_j^2\|m_j\|^2
+\frac{2c^2}{B^2}\,\mathbb{E}\|h\|^2}.
\end{equation}

\end{theorem}

\begin{proof}
Recall the batch row-gradient decomposition
\[
\bar g_j
=
\underbrace{\frac{1}{B}\sum_{b=1}^B p_j(h_b)\,h_b}_{\eqqcolon\,A_j}
\;-\;
\underbrace{\frac{1}{B}\sum_{b:y_b=j} h_b}_{\eqqcolon\,B_j}.
\]

Therefore, 
\begin{equation*}
    \begin{aligned}
        \|\bar g_j\|^2 &= \|A_j-B_j\|^2 =\|A_j\|^2+\|B_j\|^2-2\langle A_j, B_j \rangle \\
        &\ge \|A_j\|^2+\|B_j\|^2 - \left(2\|A_j\|^2+\frac{1}{2}\|B_j\|^2\right) = \frac{1}{2}\|B_j\|^2 - \|A_j\|^2,
    \end{aligned}
\end{equation*}
where the inequality follows from Young's inequality.
Taking expectations and using \eqref{ineq:cor2_mean} and \eqref{eq:weak_bg_second_moment_inline_prop}, we obtain
\begin{equation}\label{ineq:lower_bound_A}
\mathbb{E}\|\bar g_j\|^2 \ge \frac{1}{2}\mathbb{E}\|B_j\|^2 - \frac{c^2}{B^2}\mathbb{E}\|h\|^2.
\end{equation}
Next, let $v \coloneqq \mathbf{1}\{y=j\}\,h$. Then
\[\frac{1}{B}\sum_{b:y_b=j} h_b = \frac{1}{B}\sum_{b=1}^B v_b\] 
with i.i.d.\ copies $v_b$ of $v$. Therefore,
\begin{equation}\label{ineq:lower_bound_B}
\mathbb{E}\left\|B_j\right\|^2
=
\frac{1}{B^2}\Bigl(
B\,\mathbb{E}\|v\|^2 + B(B-1)\|\mathbb{E}v\|^2
\Bigr)
\ge
\frac{1}{B}\,\mathbb{E}\|v\|^2
=
\frac{1}{B}\,q_j\,\mathbb{E}[\|h\|^2\mid y=j]
=
\frac{1}{B}\,q_j s_j^2.
\end{equation}
Combining \eqref{ineq:lower_bound_A} and \eqref{ineq:lower_bound_B} gives
\begin{equation}\label{ineq:lower_bound}
    \mathbb{E}\|\bar g_j\|^2 \ge \frac{1}{2B}q_j s_j^2 - \frac{c^2}{B^2}\mathbb{E}\|h\|^2.
\end{equation}

On the other hand, by Proposition \ref{prop:grad_tail}, we have
\begin{equation}
\label{ineq:upper_bound}
\mathbb{E}\|\bar g_j\|^2
\;\le\;
\frac{2q_j}{B}s_j^2
\;+\;
2q_j^2\|m_j\|^2
\;+\;
\frac{2c^2}{B^2}\,\mathbb{E}\|h\|^2.
\end{equation}

Combining \eqref{ineq:lower_bound} and \eqref{ineq:upper_bound} gives \eqref{eq:grad_ratio_bound}.
\end{proof}

\section{Characterization of the Intermediate-Layer Gradient Norm}\label{apx:proof_hidden}

We now extend the result to intermediate layers. The proof follows
the same head--tail structure as Theorem~\ref{thm:full_grad}, with the output-layer feature vector
replaced by the class-dependent intermediate-layer matrix \(M_j^{(l)}\), which we will introduce later.

Let \((x,y)\sim\mathcal D\), where \(y\in\mathcal V:=\{1,\ldots,V\}\). For an intermediate layer $l$ with weight parameter $W^{(l)}$, let
\[
a^{(l)}=W^{(l)}h^{(l-1)},
h^{(l)}=\phi(a^{(l)}),
\]
where $\phi$ is the activation. Write
\[
h_l^-:=h^{(l-1)}.
\]

We view the final softmax probability as a function of \(h_l^-\). Specifically, for each token \(j\in\mathcal V\), define the logits and output probability with respect to token $j$ as
\[
z_j:=\left[
F_{\ge l}(h_l^-)
\right]_j, \qquad
p_j^{(l)}(h_l^-)
:=
\left[
\operatorname{softmax}
\left(
F_{\ge l}(h_l^-)
\right)
\right]_j,
\]
where \(F_{\ge l}\) denotes the remaining network mapping from the
layer-\((l-1)\) representation \(h_l^-\) to the final logits. Define also
\[
q_j^{(l)}(h_l^-)
:=
\mathbb P(y=j\mid h^{(l-1)}=h_l^-),
\qquad
q_j:=\mathbb P(y=j).
\]

For each token \(j\in\mathcal V\), define
\[
M_j^{(l)}
:=
D_{\phi}\bigl(J_j^{(l)}\bigr)^\top
\bigl(h^{(l-1)}\bigr)^\top ,
\]
where \(J_j^{(l)}\) is the \(j\)-th row of the logit Jacobian
\[
J^{(l)}
=
\frac{\partial z}{\partial h^{(l)}},
\]
and
\[
D_\phi
=
\operatorname{Diag}\bigl(\phi'(a^{(l)})\bigr).
\]

Given an i.i.d. mini-batch \(\{(x_b,y_b)\}_{b=1}^B\), let
\[
h_{b,l}^-:=h_b^{(l-1)},
\qquad
M_{b,j}^{(l)}:=M_j^{(l)}(x_b),
\qquad
p_{b,j}^{(l)}:=p_j^{(l)}(h_b^{(l-1)}).
\]
The class-\(j\) contribution to the batch-averaged gradient of \(W^{(l)}\) is
\[
\bar g_j^{(l)}
:=
\frac1B\sum_{b=1}^B
\bigl(p_{b,j}^{(l)}-\mathbf 1\{y_b=j\}\bigr)M_{b,j}^{(l)}.
\]
The full intermediate-layer gradient is
\[
G^{(l)}
:=
\nabla_{W^{(l)}} L
=
\sum_{j=1}^V \bar g_j^{(l)}.
\]

For each token \(j\), define
\[
m_{j,l}
:=
\mathbb E\!\left[M_j^{(l)}\mid y=j\right],
\qquad
s_{j,l}^2
:=
\mathbb E\!\left[
\|M_j^{(l)}\|_F^2
\mid y=j
\right].
\]

\begin{assumption}\label{ass:head_tail_hidden}
Let \(\mathcal H,\mathcal T\) be a partition of \(\mathcal V\), where
\(\mathcal H\) is the head set and \(\mathcal T\) is the tail set. Assume the
following.

\begin{enumerate}
\item \textbf{Well-trained head tokens.}
For each \(j\in\mathcal H\), there exists \(\varepsilon_{j,l}\ge 0\) such that
\[
\mathbb E\!\left[
\left\|
\left(
p_j^{(l)}(h^{(l-1)})
-
q_j^{(l)}(h^{(l-1)})
\right)
M_j^{(l)}
\right\|_F^2
\right]
\le
\varepsilon_{j,l}^2.
\]

\item \textbf{Rare tail tokens.}
There exists \(c_q\in(0,1)\) such that
\[
q_j\le c_q,
\qquad
\forall j\in\mathcal T.
\]

\item \textbf{Finite second moments.}
For all \(j\in\mathcal V\),
\[
\mathbb E\|M_j^{(l)}\|_F^2<\infty.
\]
\end{enumerate}
\end{assumption}

\begin{theorem}\label{thm:full_grad_hidden}
Under Assumption~\ref{ass:head_tail_hidden},
\[
\frac1V
\mathbb E\|G^{(l)}\|_F^2
\le
\left(1+\frac1B\right)
\sum_{j\in\mathcal H}\varepsilon_{j,l}^2
+
\frac1B
\sum_{j\in\mathcal H}
\mathbb E\|M_j^{(l)}\|_F^2
\]
\[
\hspace{2cm}
+
\left(\frac{2}{B}+2c_q\right)
\sum_{j\in\mathcal T}q_js_{j,l}^2
+
2\sum_{j\in\mathcal T}
\mathbb E\!\left[
\bigl(p_j^{(l)}(h^{(l-1)})\bigr)^2
\|M_j^{(l)}\|_F^2
\right].
\]
Equivalently, if denoting
\[
\Gamma_j^{(l)}
:=
\left\|
D_\phi\bigl(J_j^{(l)}\bigr)^\top
\right\|_2^2,
\]
then
\[
\frac1V
\mathbb E\|G^{(l)}\|_F^2
\le
\left(1+\frac1B\right)
\sum_{j\in\mathcal H}\varepsilon_{j,l}^2
+
\frac1B
\sum_{j\in\mathcal H}
\mathbb E\!\left[
\Gamma_j^{(l)}\|h^{(l-1)}\|_2^2
\right]
\]
\[
\hspace{2cm}
+
\left(\frac{2}{B}+2c_q\right)
\sum_{j\in\mathcal T}q_js_{j,l}^2
+
2\sum_{j\in\mathcal T}
\mathbb E\!\left[
\bigl(p_j^{(l)}(h^{(l-1)})\bigr)^2
\Gamma_j^{(l)}
\|h^{(l-1)}\|_2^2
\right].
\]
\end{theorem}

\textbf{Remark:} Compared to the unified formulation in Theorem \ref{thm:full_grad_mainbody_unified}, in this case, we take
\[
\xi=h^{(l-1)},
\qquad
U_j(\xi)=M_j^{(l)}
=
D_\phi (J_j^{(l)})^\top (h^{(l-1)})^\top,
\]
and choose $\kappa$ as the trivial upper bound $V$.

\begin{proof}
For notational simplicity, throughout the proof we write
\[
h_l^-:=h^{(l-1)}.
\]
Recall that
\[
p_j(h_l^-)
\]
denotes the final softmax probability of token \(j\), viewed as a function of
the layer-\((l-1)\) representation \(h_l^-\). Similarly,
\[
q_j(h_l^-)
:=
\mathbb P(y=j\mid h^{(l-1)}=h_l^-).
\]

For each token \(j\in\mathcal V\), define the class-wise contribution to the
mini-batch stochastic gradient of the \(l\)-th layer by
\[
\bar g_j^{(l)}
:=
\frac1B
\sum_{b=1}^B
\left(
p_j(h_{b}^{(l-1)})
-
\mathbf 1\{y_b=j\}
\right)
M_{b,j}^{(l)}.
\]
Then the full gradient can be decomposed as
\[
G^{(l)}
=
\sum_{j=1}^V \bar g_j^{(l)}.
\]
Unlike the output-layer case, the matrices \(\bar g_j^{(l)}\) are not supported
on disjoint rows. Therefore we use the universal Cauchy bound
\[
\left\|
\sum_{j=1}^V \bar g_j^{(l)}
\right\|_F^2
\le
V\sum_{j=1}^V
\|\bar g_j^{(l)}\|_F^2.
\]
Consequently,
\[
\frac1V\mathbb E\|G^{(l)}\|_F^2
\le
\sum_{j=1}^V
\mathbb E\|\bar g_j^{(l)}\|_F^2.
\]
It remains to bound the head-token and tail-token contributions separately.

\paragraph{Head tokens.}
Fix \(j\in\mathcal H\). Define the single-sample random matrix
\[
R_j^{(l)}
:=
\left(
p_j(h^{(l-1)})
-
\mathbf 1\{y=j\}
\right)
M_j^{(l)}.
\]
Then
\[
\bar g_j^{(l)}
=
\frac1B\sum_{b=1}^B R_{b,j}^{(l)},
\]
where \(R_{1,j}^{(l)},\ldots,R_{B,j}^{(l)}\) are i.i.d. copies of
\(R_j^{(l)}\).

Let
\[
\mu_j^{(l)}
:=
\mathbb E R_j^{(l)}.
\]
By the standard mean--variance decomposition for i.i.d. averages,
\[
\mathbb E\|\bar g_j^{(l)}\|_F^2
=
\|\mu_j^{(l)}\|_F^2
+
\frac1B
\operatorname{tr}
\left(
\operatorname{Cov}
\left(
\operatorname{vec}(R_j^{(l)})
\right)
\right).
\]
Equivalently,
\begin{equation}\label{ineq:cor1_trace_cov_hidden}
\mathbb E\|\bar g_j^{(l)}\|_F^2
=
\|\mu_j^{(l)}\|_F^2
+
\frac1B
\left(
\mathbb E\|R_j^{(l)}\|_F^2
-
\|\mu_j^{(l)}\|_F^2
\right)
\le
\|\mu_j^{(l)}\|_F^2
+
\frac1B
\mathbb E\|R_j^{(l)}\|_F^2.
\end{equation}

We now bound the two terms on the right-hand side. Conditioning on
\(h^{(l-1)}\), the matrix \(M_j^{(l)}\) and the probability
\(p_j(h^{(l-1)})\) are fixed, while
\[
\mathbb E[\mathbf 1\{y=j\}\mid h^{(l-1)}]
=
q_j(h^{(l-1)}).
\]
Therefore,
\[
\mu_j^{(l)}
=
\mathbb E\left[
\left(
p_j(h^{(l-1)})
-
q_j(h^{(l-1)})
\right)
M_j^{(l)}
\right].
\]
By Jensen's inequality and the definition of \(\varepsilon_{j,l}^2\),
\begin{equation}\label{ineq:cor2_trace_cov_hidden}
\|\mu_j^{(l)}\|_F^2
\le
\mathbb E\left[
\left\|
\left(
p_j(h^{(l-1)})
-
q_j(h^{(l-1)})
\right)
M_j^{(l)}
\right\|_F^2
\right]
=
\varepsilon_{j,l}^2.
\end{equation}

Next, conditioning again on \(h^{(l-1)}\), we have
\begin{align*}
    &\mathbb E\left[
\left(
p_j(h^{(l-1)})
-
\mathbf 1\{y=j\}
\right)^2
\,\middle|\,
h^{(l-1)}
\right] \\
=&
q_j(h^{(l-1)})
\left(
1-q_j(h^{(l-1)})
\right)
+
\left(
p_j(h^{(l-1)})
-
q_j(h^{(l-1)})
\right)^2.
\end{align*}

Hence
\begin{align*}
\mathbb E\|R_j^{(l)}\|_F^2
=&
\mathbb E\left[
q_j(h^{(l-1)})
\left(
1-q_j(h^{(l-1)})
\right)
\|M_j^{(l)}\|_F^2
\right]\\
&+
\mathbb E\left[
\left\|
\left(
p_j(h^{(l-1)})
-
q_j(h^{(l-1)})
\right)
M_j^{(l)}
\right\|_F^2
\right].
\end{align*}
Using the definition of \(\varepsilon_{j,l}^2\), this gives
\begin{equation}\label{ineq:cor3_trace_cov_hidden}
\mathbb E\|R_j^{(l)}\|_F^2
\le
\mathbb E\left[
q_j(h^{(l-1)})
\left(
1-q_j(h^{(l-1)})
\right)
\|M_j^{(l)}\|_F^2
\right]
+
\varepsilon_{j,l}^2.
\end{equation}
Combining \eqref{ineq:cor1_trace_cov_hidden}, \eqref{ineq:cor2_trace_cov_hidden} and \eqref{ineq:cor3_trace_cov_hidden}, we have
\[
\mathbb E\|\bar g_j^{(l)}\|_F^2
\le
\left(1+\frac1B\right)\varepsilon_{j,l}^2
+
\frac1B
\mathbb E\left[
q_j(h^{(l-1)})
\left(
1-q_j(h^{(l-1)})
\right)
\|M_j^{(l)}\|_F^2
\right].
\]
Since
\[
0\le q_j(h^{(l-1)})
\left(
1-q_j(h^{(l-1)})
\right)
\le 1,
\]
we further obtain
\[
\mathbb E\|\bar g_j^{(l)}\|_F^2
\le
\left(1+\frac1B\right)\varepsilon_{j,l}^2
+
\frac1B
\mathbb E\|M_j^{(l)}\|_F^2.
\]
Noticing that
$
\|M_j^{(l)}\|_F^2
=
\Gamma_j^{(l)}\|h^{(l-1)}\|^2,
$
we get
\[
\mathbb E\|\bar g_j^{(l)}\|_F^2
\le
\left(1+\frac1B\right)\varepsilon_{j,l}^2
+
\frac1B
\mathbb E\left[
\Gamma_j^{(l)}\|h^{(l-1)}\|^2
\right].
\]
Summing over \(j\in\mathcal H\) yields
\begin{equation}\label{ineq:head_hidden}
\sum_{j\in\mathcal H}
\mathbb E\|\bar g_j^{(l)}\|_F^2
\le
\left(1+\frac1B\right)
\sum_{j\in\mathcal H}\varepsilon_{j,l}^2
+
\frac1B
\sum_{j\in\mathcal H}
\mathbb E\left[
\Gamma_j^{(l)}\|h^{(l-1)}\|^2
\right].
\end{equation}

\paragraph{Tail tokens.}
Fix \(j\in\mathcal T\). The class-wise contribution admits the
background--hit decomposition
\[
\bar g_j^{(l)}
=
\underbrace{
\frac1B
\sum_{b=1}^B
p_j(h_b^{(l-1)})M_{b,j}^{(l)}
}_{=:A_j^{(l)}}
-
\underbrace{
\frac1B
\sum_{b:y_b=j}
M_{b,j}^{(l)}
}_{=:B_j^{(l)}}.
\]
By the elementary inequality
\[
\|A-B\|_F^2
\le
2\|A\|_F^2+2\|B\|_F^2,
\]
we have
\[
\mathbb E\|\bar g_j^{(l)}\|_F^2
\le
2\mathbb E\|A_j^{(l)}\|_F^2
+
2\mathbb E\|B_j^{(l)}\|_F^2.
\]

We first bound the background term. Let
\[
X_j^{(l)}
:=
p_j(h^{(l-1)})M_j^{(l)}.
\]
Then
\[
A_j^{(l)}
=
\frac1B\sum_{b=1}^B X_{b,j}^{(l)}.
\]
Since the samples are i.i.d.,
\[
\mathbb E\|A_j^{(l)}\|_F^2
=
\frac1B\mathbb E\|X_j^{(l)}\|_F^2
+
\frac{B-1}{B}
\|\mathbb E X_j^{(l)}\|_F^2
\le
\mathbb E\|X_j^{(l)}\|_F^2.
\]
Thus
\begin{equation}\label{ineq:background_hidden}
    \mathbb E\|A_j^{(l)}\|_F^2
\le
\mathbb E\left[
p_j(h^{(l-1)})^2
\|M_j^{(l)}\|_F^2
\right] = \mathbb E\left[
p_j(h^{(l-1)})^2
\Gamma_j^{(l)}
\|h^{(l-1)}\|^2
\right].
\end{equation}

We next bound the label-hit term. Define
\[
Y_j^{(l)}
:=
\mathbf 1\{y=j\}M_j^{(l)}.
\]
Then
\[
B_j^{(l)}
=
\frac1B\sum_{b=1}^B Y_{b,j}^{(l)}.
\]
Therefore,
\[
\mathbb E\|B_j^{(l)}\|_F^2
=
\frac1B\mathbb E\|Y_j^{(l)}\|_F^2
+
\frac{B-1}{B}
\|\mathbb E Y_j^{(l)}\|_F^2.
\]
By definition,
\[
\mathbb E\|Y_j^{(l)}\|_F^2
=
\mathbb E\left[
\mathbf 1\{y=j\}
\|M_j^{(l)}\|_F^2
\right]
=
q_j
\mathbb E\left[
\|M_j^{(l)}\|_F^2
\mid y=j
\right]
=
q_j s_{j,l}^2.
\]
Moreover,
\[
\mathbb E Y_j^{(l)}
=
\mathbb E\left[
\mathbf 1\{y=j\}M_j^{(l)}
\right]
=
q_j
\mathbb E\left[
M_j^{(l)}
\mid y=j
\right].
\]
Hence
\[
\|\mathbb E Y_j^{(l)}\|_F^2
=
q_j^2
\left\|
\mathbb E\left[
M_j^{(l)}
\mid y=j
\right]
\right\|_F^2.
\]
By Jensen's inequality,
\[
\left\|
\mathbb E\left[
M_j^{(l)}
\mid y=j
\right]
\right\|_F^2
\le
\mathbb E\left[
\|M_j^{(l)}\|_F^2
\mid y=j
\right]
=
s_{j,l}^2.
\]
Therefore,
\[
\|\mathbb E Y_j^{(l)}\|_F^2
\le
q_j^2s_{j,l}^2.
\]
Since \(j\in\mathcal T\) and \(c_q=\max_{j\in\mathcal T}q_j\), we have
\(q_j\le c_q\). Thus
\[
q_j^2s_{j,l}^2
\le
c_q q_j s_{j,l}^2.
\]
Consequently,
\begin{equation}\label{ineq:hit_hidden}
\mathbb E\|B_j^{(l)}\|_F^2
\le
\frac{q_j}{B}s_{j,l}^2
+
c_q q_j s_{j,l}^2
=
\left(\frac1B+c_q\right)q_j s_{j,l}^2.
\end{equation}

Combining \eqref{ineq:background_hidden} and \eqref{ineq:hit_hidden} gives
\[
\mathbb E\|\bar g_j^{(l)}\|_F^2
\le
2
\mathbb E\left[
p_j(h^{(l-1)})^2
\Gamma_j^{(l)}
\|h^{(l-1)}\|^2
\right]
+
\left(\frac{2}{B}+2c_q\right)q_j s_{j,l}^2.
\]
Summing over \(j\in\mathcal T\), we obtain
\begin{equation}\label{ineq:tail_hidden}
\sum_{j\in\mathcal T}
\mathbb E\|\bar g_j^{(l)}\|_F^2
\le
2
\sum_{j\in\mathcal T}
\mathbb E\left[
p_j(h^{(l-1)})^2
\Gamma_j^{(l)}
\|h^{(l-1)}\|^2
\right]
+
\left(\frac{2}{B}+2c_q\right)
\sum_{j\in\mathcal T}
q_j s_{j,l}^2.
\end{equation}

\paragraph{Combining the head and tail bounds.}
Putting the head and tail estimates \eqref{ineq:head_hidden} and \eqref{ineq:tail_hidden} together, we have
\begin{align*}
\frac1V\mathbb E\|G^{(l)}\|_F^2 
\le &
\sum_{j=1}^V
\mathbb E\|\bar g_j^{(l)}\|_F^2
=
\sum_{j\in\mathcal H}
\mathbb E\|\bar g_j^{(l)}\|_F^2
+
\sum_{j\in\mathcal T}
\mathbb E\|\bar g_j^{(l)}\|_F^2 \\
\le &
\left(1+\frac{1}{B}\right)
\sum_{j\in\mathcal H}
\varepsilon_{j,l}^2
+
\frac{1}{B}
\sum_{j\in\mathcal H}
\mathbb E\left[
\Gamma_j^{(l)}
\|h^{(l-1)}\|^2
\right] \\
&+
\left(\frac{2}{B}+2c_q\right)
\sum_{j\in\mathcal T}
q_j s_{j,l}^2
+
2
\sum_{j\in\mathcal T}
\mathbb E\left[
p_j(h^{(l-1)})^2
\Gamma_j^{(l)}
\|h^{(l-1)}\|^2
\right].
\end{align*}
The proof is complete.
\end{proof}

\end{document}